\documentclass[a4paper,fleqn]{cas-sc}

\usepackage[numbers]{natbib}
\usepackage{subcaption}
\usepackage{amsthm}

\def\tsc#1{\csdef{#1}{\textsc{\lowercase{#1}}\xspace}}
\tsc{WGM}
\tsc{QE}

\newtheorem{theorem}{Theorem}
\newtheorem{lemma}[theorem]{Lemma}
\newtheorem{definition}{Definition}[section]

\begin{document}
\let\WriteBookmarks\relax
\def\floatpagepagefraction{1}
\def\textpagefraction{.001}

\shorttitle{Windowed Fourier Propagator}    

\shortauthors{Y. Cai et~al}  

\title [mode = title]{Windowed Fourier Propagator: A Frequency-Local Neural Operator for Wave Equations in Inhomogeneous Media}  



%

\author[1]{Yiyang Cai}[orcid=0009-0001-1519-6910]


\credit{Conceptualization, Methodology,  Formal analysis, Writing}
\affiliation[1]{organization={School of Mathematical Sciences, Fudan University},    
            city={Shanghai},
            postcode={200433}, 
            country={China}}

\author[1]{Zixuan Qiu}
\credit{Data curation}

\author[2]{Yunlu Shu}
\credit{Investigation}

\author[1]{Jiamao Wu}
\credit{Data curation}

\author[1,3]{Yingzhou Li}
\cormark[1]
\ead{yingzhouli@fudan.edu.cn}
\credit{Supervision, Resources, Project administration}

\author[2,3]{Tianyu Wang}
\ead{wangtianyu@fudan.edu.cn}
\cormark[1]
\credit{Supervision}

\author[1,2,3]{Xi Chen}
\ead{xi_chen@fudan.edu.cn}
\cormark[1]
\credit{Supervision,Validation, Conceptualization, Funding acquisition}

\affiliation[2]{organization={Shanghai Center for Mathematical Sciences, Fudan University},    
            city={Shanghai},
            postcode={200438}, 
            country={China}}

\affiliation[3]{organization={Center for Applied Mathematics, Fudan University},    
            city={Shanghai},
            postcode={200438}, 
            country={China}}

\cortext[1]{Corresponding author}


\begin{abstract}
Wave equations are fundamental to describing a vast array of physical phenomena, yet their simulation in inhomogeneous media poses a computational challenge due to the highly oscillatory nature of the solutions. To overcome the high costs of traditional solvers, we propose the Windowed Fourier Propagator (WFP), a novel neural operator that efficiently learns the solution operator.  
The WFP's design is rooted in the physical principle of frequency locality, where wave energy scatters primarily to adjacent frequencies. By learning a set of compact, localized propagators, each mapping an input frequency to a small `window' of outputs, our method avoids the complexity of dense interaction models and achieves computational efficiency. Another key feature is the explicit preservation of superposition, which enables remarkable generalization from simple training data (e.g., plane waves) to arbitrary, complex wave states. We demonstrate that the WFP provides an explainable, efficient and accurate framework for data-driven wave modeling in complex media.
\end{abstract}



\begin{highlights}
\item Achieves efficiency via windowed, frequency-local interactions.
\item Preserves linearity for generalization to unseen states. 
\item Frequency-domain operation enables resolution-invariant learning. 
\item Dual-pathway architecture separates medium and wave state processing.

\end{highlights}


\maketitle

\section{Introduction}\label{introduction}
Wave equations describe various wave phenomena across science and technology, including acoustic, electromagnetic, elastic, and gravitational waves. The fast computation of wave equations holds significant value for quantitative analysis, prediction and detection tasks in both theoretical research and engineering applications. 

To address the challenge of resolving highly oscillatory solutions, a variety of advanced numerical methods have been developed. Among them, spectral methods, and particularly pseudo-spectral methods~\cite{orszag1971numerical,trefethen2000spectral}, 
are noted for their exponential convergence properties when applied to smooth solutions. This characteristic allows them to achieve high accuracy with fewer grid points compared to local methods such as finite elements or finite differences.
 However, all these methods are bound by a classic trade-off: explicit schemes require small time steps due to CFL conditions, while stable implicit schemes incur a high computational cost per step. This fundamental limitation motivates the search for new computational paradigms.

Another important line of research focuses on designing more efficient propagators in the frequency or phase space. For instance, Qian and Ying \cite{qian2010fast} proposed a method based on the fast Gaussian wavepacket transform, which efficiently solves the wave equation with a complexity of $O(N^d \log N)$ by propagating Gaussian wavepackets in phase space. Both pseudo-spectral methods and these advanced techniques underscore the core advantage of operating in the frequency or phase domain. Building upon this foundation, the integration of frequency-domain principles with modern deep learning architectures opens up new avenues for learning the solution operator directly, leading to novel solvers that merge insights from spectral methods with the power of machine learning.

\subsection{Related Works}
\label{sec:related}
Our work is positioned within the broader landscape of scientific machine learning. There are many related works, which are categrized into a few groups below.

\textbf{Physics-Informed Neural Networks (PINNs).}
PINNs~\cite{raissi2019physics} are a foundational deep learning approach for solving PDEs. The core methodology embeds the governing physical laws directly into the loss function by minimizing the PDE residual, a process enabled by automatic differentiation. 
This mesh-free technique offers remarkable flexibility for problems where training data is specified solely on the domain boundaries or as sparse, unstructured measurements within the interior.
 The PINN framework has proven highly versatile, with successful extensions to complex applications such as domain decomposition~\cite{jagtap2020extended}, conservation laws~\cite{jagtap2020conservative}, and high-speed flows~\cite{mao2020physics}. Supported by robust software libraries~\cite{lu2021deepxde} and a growing body of theoretical analysis on its training dynamics~\cite{wang2021understanding,xu2020frequency}, PINNs have become a versatile tool in scientific computing.

\textbf{Deep Operator Networks (DeepONet).}
DeepONet~\cite{lu2021learning} are a foundational framework for learning operators between function spaces, grounded in universal approximation theorems~\cite{chen1995universal,li2020variational}. The architecture uniquely employs a dual-network structure: a ``branch net'' processes the input function, while a ``trunk net'' processes the output domain coordinates. Their outputs are combined to predict the solution pointwise. This versatile design has been successfully extended and applied to a wide array of problems in science and engineering~\cite{anandkumar2020neural, brandstettermessage, jin2022mionet,montes2021accelerating}, including in physics-informed~\cite{li2024physics, wang2021learning} and frequency-domain contexts~\cite{zhu2023fourier}.

\textbf{CNN-based Models.}
Convolutional Neural Networks (CNNs) serve as a powerful foundation for data-driven PDE solvers due to their inherent ability to capture local spatial patterns. Architectures like PDE-Net~\cite{Li2020,long2019pde,long2018pde} leverage this by designing convolutional filters that explicitly learn discretized differential operators from data. 
Beyond learning local operators, CNNs have also been adapted to approximate non-local solution operators (e.g., Green's functions) by leveraging mathematical structures such as hierarchical matrices. 
Specifically, Multiscale Neural Networks~\cite{fan2019mnnh2,fan2019mnn,feliu2020meta} utilize multi-scale wavelet decompositions to capture sparse interactions at different scales using local neural networks. 
In the context of solving PDEs, these models excel at learning spatial structures and have demonstrated strong performance in various simulation tasks.

\textbf{Fourier Neural Operator (FNO).} 
The Fourier Neural Operator (FNO~\cite{li2023solving,lifourier,li2024physics}) is a prominent deep learning framework for solving PDEs. Its core architectural innovation is to parametrize the integral kernel directly in the frequency domain. This is achieved by applying a Fourier transform, performing a linear transformation on the lower-frequency modes, and then applying an inverse Fourier transform. This frequency-domain operation allows FNO to efficiently model global dependencies and learn resolution-invariant solution operators. To handle parametric PDEs, spatially varying coefficients are typically incorporated as additional input channels. For time-dependent problems, the FNO learns an autoregressive mapping between consecutive time steps. The architecture has demonstrated strong performance and computational efficiency across a range of benchmark problems~\cite{guibas2021adaptive,li2023fourier,tranfactorized}.

\subsection{The Method and Contributions}

In this paper, we introduce the Windowed Fourier Propagator (WFP), 
a novel neural operator framework designed to overcome the 
significant computational challenges in solving time-domain wave equations, 
particularly in high-frequency regimes and inhomogeneous media. 
The WFP architecture is not a black-box model, 
instead, its design is directly motivated by a rigorous mathematical analysis of 
the underlying wave physics, allowing it to achieve a unique blend of efficiency, 
accuracy, and physical consistency.

We shall demonstrate that the solution to the wave equation exhibits a phenomenon, termed \textbf{frequency locality}, see theorem~\ref{thm:main} for precise quantitative estimates. 
Specifically, if an initial single-frequency wave is prescribed for the wave equation in inhomogeneous media, the singularities will be propagated along Hamiltonian flows in the phase space and the energy transfer is predominantly confined to a 
small frequency neighbourhood around the initial frequency. This phenomenon can be viewed as the finite speed propagation of waves in the frequency domain.  
Consequently, a local model capturing only the interactions within a small frequency neighborhood is sufficient to characterize the overall evolution of the system.

We designed the WFP to explicitly exploit frequency locality. 
The idea is to learn a set of local evolution operators (or propagators) in the frequency domain, 
rather than a single global one. 
The overall workflow of our method can be summarized as follows:
\begin{enumerate}
    \item \textbf{Input Decomposition and Feature Extraction:} 
The initial conditions and the medium velocity field are transformed into the frequency domain. We then extract a sparse set of dominant initial frequencies, along with a compact, resolution-invariant feature vector that characterizes the medium.
    \item \textbf{Localized Evolution via a Neural Network:} A core neural network module learns
    the mapping from a driving frequency component and the medium features to a local propagator. 
    This propagator is a small, learnable dictionary describing how the driving frequency transfers
    its energy into its immediate frequency neighborhood, or `window'.
    \item \textbf{Synthesis of the Evolved Spectrum:} The evolution of each driving frequency component is computed by scaling its corresponding local propagator with its initial complex amplitude. The complete Fourier coefficients of the final solution are then synthesized by aggregating contributions from all initial frequency windows.
    \item \textbf{Final Reconstruction:} Finally, a single inverse  Fourier transform 
    reconstructs the full-field solution in the spatial domain.
\end{enumerate}

This structured, physics-informed approach allows the WFP's core neural network component to operate with a computational complexity of $O(N \cdot r^d)$, where $N$ is the number of dominant input frequencies and $r$ is the small window radius. This complexity, which does not account for the initial and final Fourier transforms, is determined by hyperparameters $r$ tuned to the physical nature of the problem and the desired evolution time.

The contributions of this framework are threefold:

\begin{enumerate}
    \item \textbf{A Computationally Efficient and Resolution-Invariant Architecture.}
    
The WFP learns a set of localized propagators, each directly mapping a dominant frequency component to its evolved state within mitigate the strict CFL constraints of traditional explicit solvers. Moreover, the core network operates on fixed-size frequency-domain representations, decoupling computational cost from the spatial grid resolution. This property makes the WFP a resolution-invariant and highly scalable framework, as a model trained on a coarse grid can be seamlessly deployed on high-resolution domains without retraining.

    \item \textbf{A Principled and Generalizable Operator via Explicit Physical Modelling.}

   A common strategy for parametric neural operators is to concatenate the initial state and medium properties as input channels. Our framework explores a different approach by decoupling these components to explicitly model the underlying physics. The WFP learns an evolution operator that is conditioned solely on the medium's frequency-domain features. This operator is then applied linearly to the initial state by scaling it with the initial wave amplitudes, a design that preserves the principle of superposition. This provides a powerful inductive bias, leading to superior generalization. By learning an operator independent of the initial amplitudes, the WFP can be trained on a simple basis of plane waves and yet accurately predict the evolution of complex, unseen initial states, significantly enhancing its data efficiency and robustness.

    \item \textbf{A Highly Parallelizable and Storage-Efficient Network.}
    
The training of end-to-end neural operators often involves generating and storing datasets of high-resolution, full-field snapshots. Our work focuses on refining this data handling and computational workflow for greater efficiency. Our framework is designed to be adaptive by operating on compact vectors representing only the dominant Fourier modes of the initial state, rather than full spatial grids. Building on this, we treat the evolution of each dominant mode as an independent sub-problem. This approach offers several advantages: dataset storage becomes minimal, the computational cost naturally adapts to the input's complexity, and the entire process is parallelizable. This structure is perfectly suited for modern hardware accelerators, enabling massive throughput for large-scale, multi-frequency problems.

\end{enumerate}

\begin{figure}[h]
    \centering
    \begin{subfigure}[b]{0.32\textwidth}
        \centering
        \includegraphics[page=2,width=\textwidth]{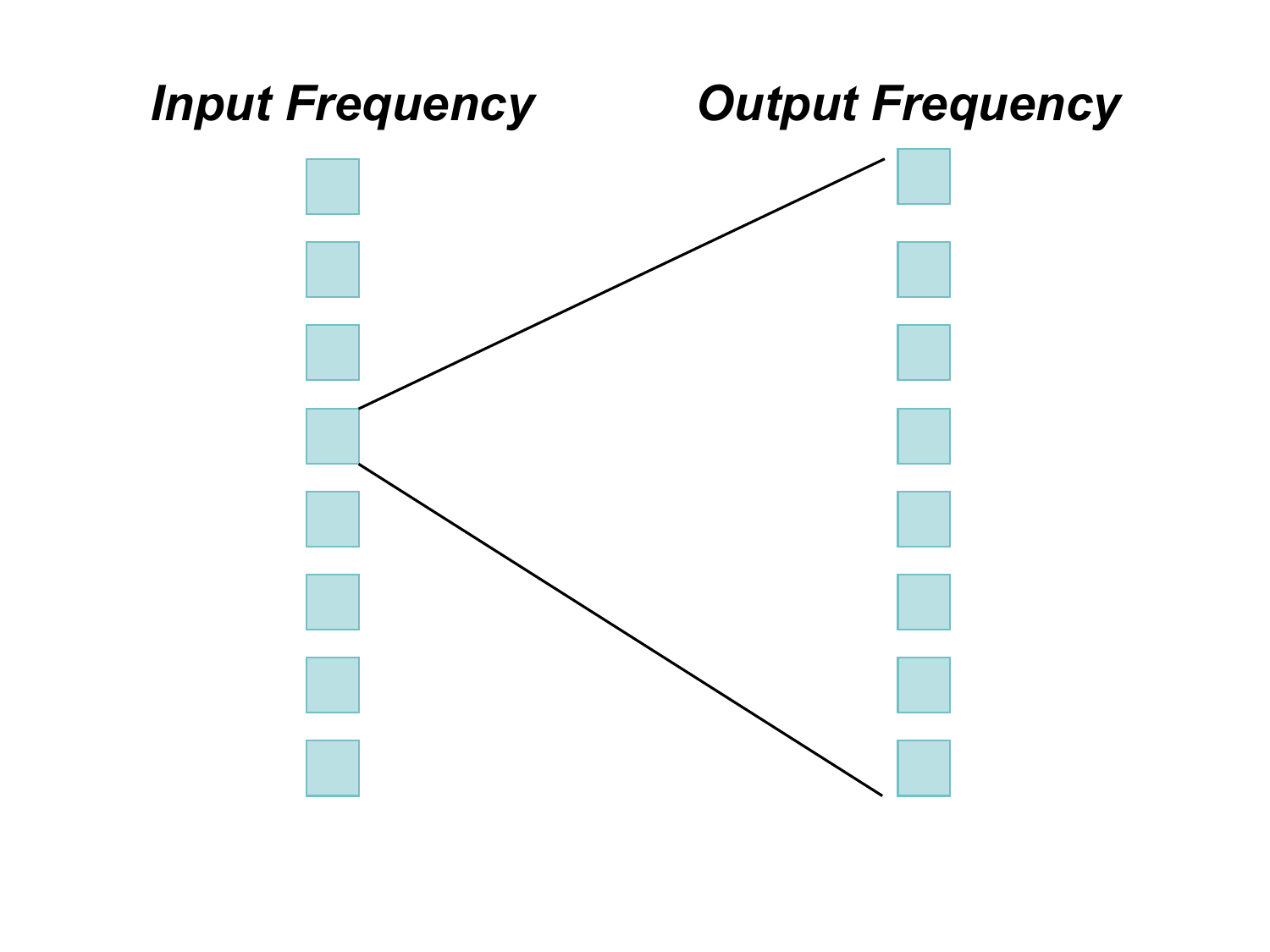}
        \caption{Fully Connected Interaction}
        \label{fig:linear_interaction}
    \end{subfigure}
    \hfill
    \begin{subfigure}[b]{0.32\textwidth}
        \centering
        \includegraphics[page=4,width=\textwidth]{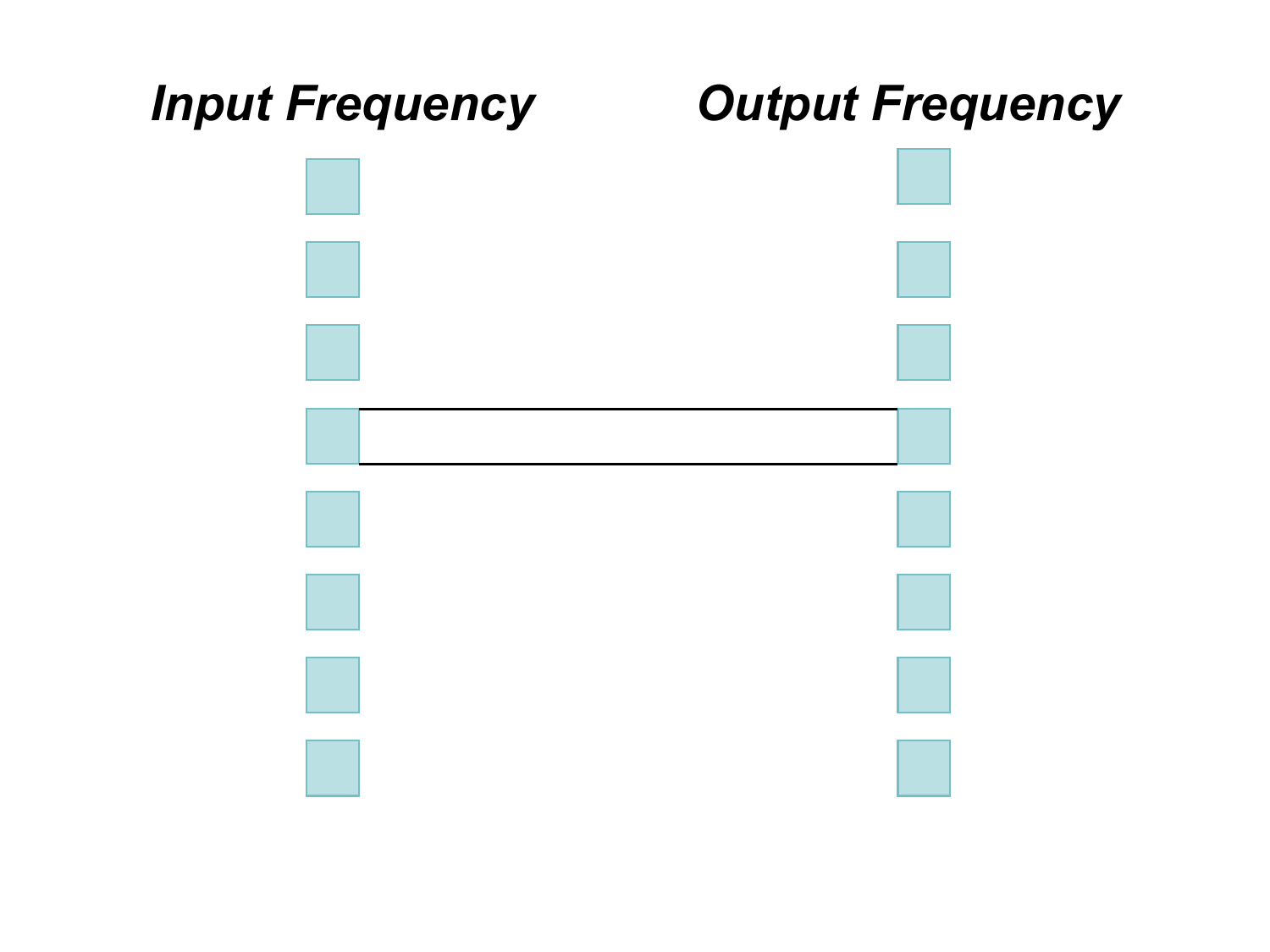}
        \caption{Standard FNO Interaction}
        \label{fig:elementfno_interaction}
    \end{subfigure}
    \hfill
    \begin{subfigure}[b]{0.32\textwidth}
        \centering
        \includegraphics[page=3,width=\textwidth]{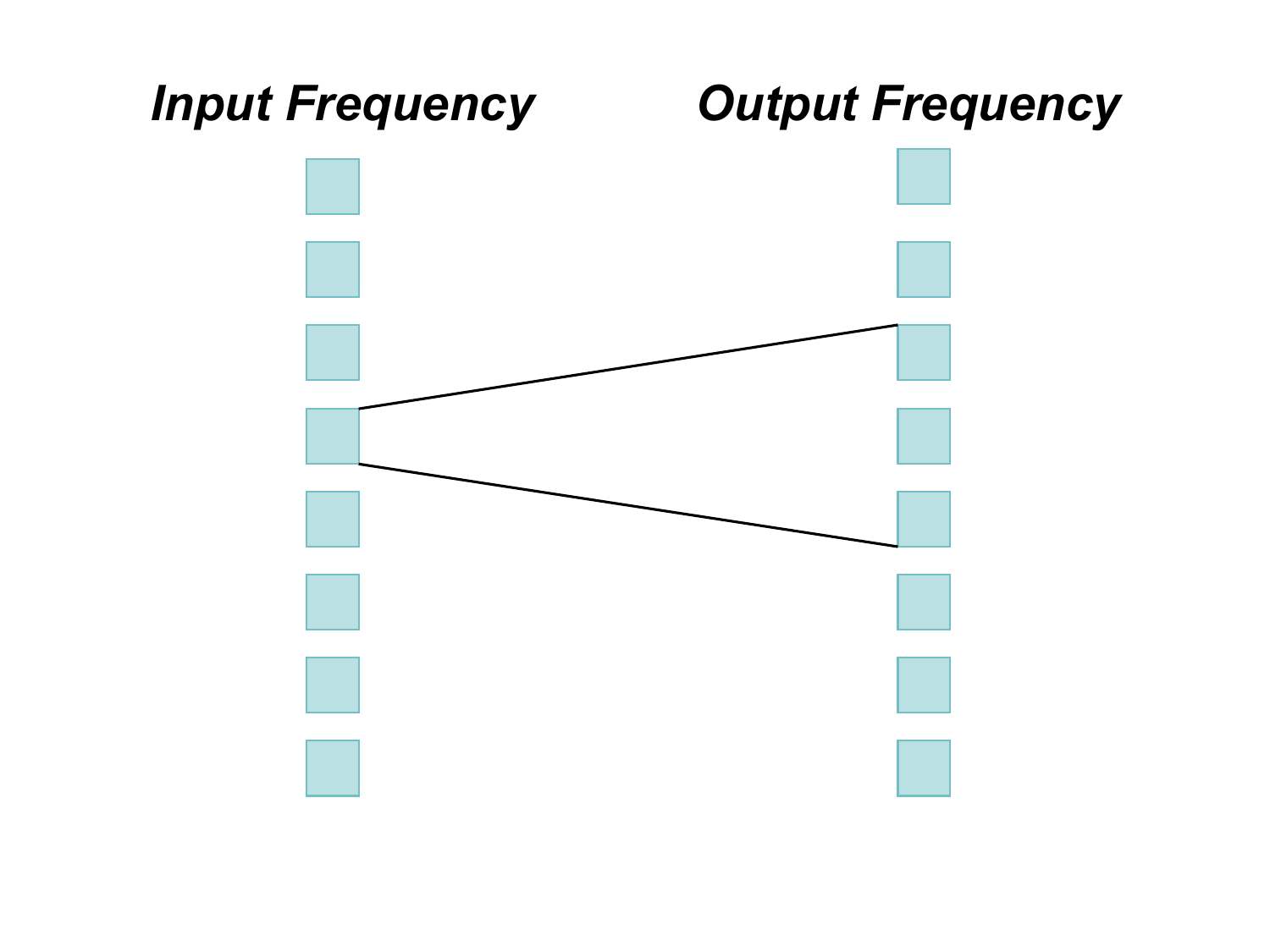}
        \caption{WFP Interaction}
        \label{fig:wfp_interaction}
    \end{subfigure}
    
    \raggedright

    \caption{\textbf{Comparison of Frequency Interaction Patterns:} 
    The left panel shows the fully connected network with full frequency coupling, 
    the middle panel illustrates the standard FNO with independent and elementwise mode processing, 
    the right panel demonstrates the WFP's localized windowed interactions.}
    \label{fig:frequency_interactions}
\end{figure}
\textit{\textbf{Paper organization. }} The rest of this paper is organized as follows: Section~\ref{sec:problem_formulation} reviews some properties on wave equations. Section~\ref{sec:analysis} presents theoretical analysis of frequency locality. Section~\ref{sec:wfp_methodology} introduces our Windowed Fourier Propagator (WFP) method. Section~\ref{sec:experiments} validates WFP through numerical experiments. We conclude the paper with a discussion in Section~\ref{sec:conclusion}.





\section{The wave equation}\label{sec:problem_formulation}

Let $\Omega := \mathbb{T}^d = (\mathbb{R}/2\pi\mathbb{Z})^d$ be the $d$-dimensional torus (equivalent to the domain $[0, 2\pi]^d$ with periodic boundary conditions). We consider the wave equation in $\Omega$:
\begin{equation}\label{eq:wave_equation_homogeneous} \left\{
	\begin{aligned}
  \partial^2_t u(\mathbf{x},t) - \nabla \cdot \left( c(\mathbf{x})^2 \nabla u(\mathbf{x},t) \right) &= 0, && \text{for } (\mathbf{x},t) \in \Omega \times \mathbb{R}_+\\
u(\mathbf{x},0) &= f(\mathbf{x}), && \text{for } \mathbf{x} \in \Omega   \\ 
 \partial_t u(\mathbf{x},0) &= g(\mathbf{x}), && \text{for } \mathbf{x} \in \Omega  
	\end{aligned}\right. ,
\end{equation}
where $u(\mathbf{x},t)$ is the scalar wave field, $c(\mathbf{x})$ is the spatially varying wave speed, and $f(\mathbf{x})$ and $g(\mathbf{x})$ represent the initial displacement and velocity profiles, respectively.

\subsection{The principle of superposition}
 
A fundamental property of linear wave equations is that they adhere to the principle of superposition. This implies that the solution to the wave equation with a linear combination of initial conditions is equal to the linear combination of the solutions corresponding to each individual initial condition.

Any function $h \in L^2(\Omega)$ can be expanded in terms of the standard orthonormal Fourier basis. We define the normalized basis functions as:
\begin{equation}
    \phi_{\mathbf{k}}(\mathbf{x}) := (2\pi)^{-d/2} e^{\imath \mathbf{k} \cdot \mathbf{x}}, \quad \mathbf{k} \in \mathbb{Z}^d.
\end{equation}
These functions satisfy the orthonormality relation $\langle \phi_{\mathbf{k}}, \phi_{\mathbf{m}} \rangle = \delta_{\mathbf{k},\mathbf{m}}$. Consequently, the Fourier series representation of $h$ is given by:
\begin{equation}\label{eqn : Fourier series}
	h(\mathbf{x}) = \sum_{\mathbf{k} \in \mathbb{Z}^d} \hat{h}_{\mathbf{k}} \phi_{\mathbf{k}}(\mathbf{x}),
\end{equation} 
where the Fourier coefficient $\hat{h}_{\mathbf{k}}$ is defined as the inner product of $h$ with the basis function:
\begin{equation}\label{eqn : Fourier coefficient}
    \hat{h}_{\mathbf{k}} := \langle h, \phi_{\mathbf{k}} \rangle = \int_{\Omega} h(\mathbf{x}) \overline{\phi_{\mathbf{k}}(\mathbf{x})}\,\mathrm{d}\mathbf{x}.
\end{equation}

By the principle of superposition, solving Equation~\eqref{eq:wave_equation_homogeneous} for general initial data $(f, g)$ reduces to analyzing the evolution of individual Fourier modes. Specifically, it suffices to consider the case with a single-frequency plane wave initial condition:
\begin{equation}
    u(\mathbf{x}, 0) = e^{\imath \mathbf{k}_0 \cdot \mathbf{x}}, \quad \partial_t u(\mathbf{x}, 0) = 0,
\end{equation}
where $\mathbf{k}_0 \in \mathbb{Z}^d$ is termed the \textbf{driving frequency}. The linearity of the equation ensures that the full solution can be reconstructed by combining the evolved fields of these single-frequency components weighted by their corresponding Fourier coefficients $\hat{f}_{\mathbf{k}}$ and $\hat{g}_{\mathbf{k}}$.

\subsection{Propagation of singularities}
The singularities of waves amount to the energies of waves. Each Fourier (frequency) component in Equation~\eqref{eqn : Fourier series} is thought of as an energy carrier and the Fourier coefficient Equation~\eqref{eqn : Fourier coefficient} measures the magnitude of singularities at $\mathbf{m}$. The modes with non-zero (zero) amplitude are called excited (unexcited) modes.

The classical theorem on propagation of singularities, due to \cite{ FIO2,FIO1}, states that the singularities (nonsmoothness) of waves, in the phase space, propagate, as time evolves, along the Hamiltonian flows of the principal symbol of the wave operator. 
In our case, the wave operator is $P := \partial_t^2 - \nabla \cdot (c^2(\mathbf{x}) \nabla)$, which has the principal symbol form:
\[
\sigma(P) := - \tau^2 + c^2(\mathbf{x}) |\boldsymbol{\xi}|^2 , \quad \mbox{for $(t, \mathbf{x}, \tau, \boldsymbol{\xi}) \in T^\ast  \mathbb{R}^{d + 1}$}.
\]  
Here, $T^\ast \mathbb{R}^{d + 1}$ denotes the phase space, where each point consists of a spacetime coordinate $(t, \mathbf{x})$ and its corresponding frequency-wavevector pair $(\tau, \boldsymbol{\xi})$. Hamiltonian flows are the integral curves along the Hamiltonian vector field of $\sigma(P)$, which, in local coordinates, reads 
\begin{equation}\label{eqn : Hamilton v.f.}
H_{\sigma(P)} : = -2\tau \partial_{t} +  2 c^2(\mathbf{x}) \boldsymbol{\xi} \cdot \partial_{\mathbf{x}} - 2 |\boldsymbol{\xi}|^2  c(\mathbf{x})   \nabla c(\mathbf{x})\cdot \partial_{\boldsymbol{\xi}}.
\end{equation}
Owing to the presence of $\nabla c(\mathbf{x})$, the singularities propagate not only in the physical space but also in the frequency domain. Namely, the energies transfer among Fourier modes.

Governed by Equation~\eqref{eqn : Hamilton v.f.}, any Hamiltonian flow propagates, in the phase space $T^\ast  \mathbb{R}^{d}$, at the velocity
\[  
\frac{\mathrm{d}(\mathbf{x}, \boldsymbol{\xi})}{\mathrm{d}t}  = \frac{1}{\tau} \left( c^2(\mathbf{x}) \boldsymbol{\xi}, -  |\boldsymbol{\xi}|^2  c(\mathbf{x})   \nabla c(\mathbf{x})\right).
\] 
Since $\tau^2 = c^2(\mathbf{x})|\boldsymbol{\xi}|^2 $ on the flow, this implies a well-known fact, that the waves propagate in the physical space at a finite speed $c(\mathbf{x})$. Moreover, the evolution of the frequency $\boldsymbol{\xi}$ (i.e. the transferring among Fourier modes) obeys 
\[
\frac{\mathrm{d}\boldsymbol{\xi}}{\mathrm{d}t} = \pm \nabla c(\mathbf{x}) |\boldsymbol{\xi}|.
\] 
For any initial value $\boldsymbol{\xi}_0$, the propagation speed of $\boldsymbol{\xi}$ in the frequency domain is dominated by $\nabla c$ in finite times.

We may build some intuition with a numerical experiment. A 1D single-frequency plane wave, $u(x,0) = \sin(2\pi k_x x)$, is propagated through a weakly inhomogeneous medium using a numerical solver. We visualize the evolution of the Fourier spectrum  of the solution over time in Figure~\ref{fig:freq_locality_3d}. The experiment reveals two crucial phenomena: first, at any given time, the energy remains sharply concentrated in the frequency neighborhood of the initial mode; second, this region of excited frequencies gradually expands over time. This observation of energy spreading at a finite speed in the frequency domain is the cornerstone of our proposed method.

\begin{figure}[h!]
    \centering
    \includegraphics[width=0.4\textwidth]{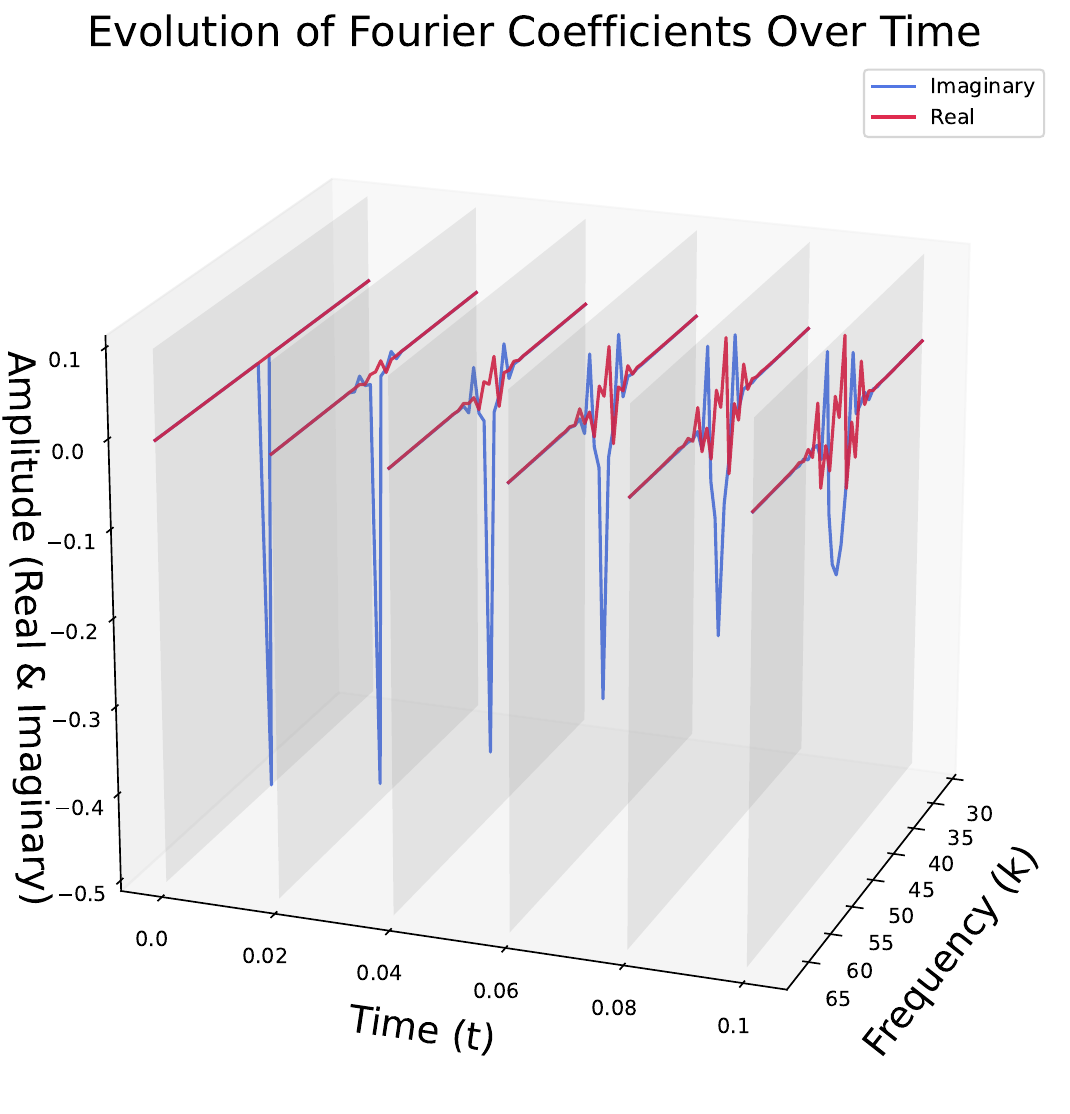}
    \caption{Numerical evidence of frequency locality. The plot shows the Fourier spectrum of the wave solution at different time slices. Initially, only a single frequency is excited (back plane). As time progresses, energy spreads to adjacent frequencies, but remains highly localized in a growing area.}
    \label{fig:freq_locality_3d}
\end{figure}

\section{Energy transfer among frequency components}\label{sec:analysis}

This section delves into a pivotal question: 
how does the energy of a single-frequency wave 
transfer among other frequency modes as it propagates through an inhomogeneous medium?
Motivated by the numerical phenomenon, we establish  quantitative estimates for the spreading of amplitudes among Fourier modes.

\subsection{Problem Setup and Preliminaries}

We investigate the divergence-form wave equation on the $d$-dimensional torus $\Omega = \mathbb{T}^d \cong [0, 2\pi]^d$:
\begin{equation} \label{eq:wave_eq}
    \partial_t^2 u - \nabla \cdot (c^2(\mathbf{x}) \nabla u) = 0, \quad u(\mathbf{x}, 0) = \phi_{\mathbf{k}_0}(\mathbf{x}), \quad \partial_t u(\mathbf{x}, 0) = 0.
\end{equation}
The medium coefficient is $c^2(\mathbf{x})=1+\epsilon p(\mathbf{x})$, where $p$ is a smooth periodic function with zero mean and $\|p\|_\infty \le 1$. We assume bounds $m \le c^2(\mathbf{x}) \le M$, specifically $m = 1-\epsilon$ and $M=1+\epsilon$. The initial condition is a single frequency mode $\mathbf{k}_0$, where $\phi_{\mathbf{n}}(\mathbf{x}) = (2\pi)^{-d/2} e^{i \mathbf{n} \cdot \mathbf{x}}$ denotes the standard Fourier basis.

Let $\hat{u}_{\mathbf{n}}(t)$ be the Fourier coefficient of $u(\cdot, t)$, defined as 
\begin{equation}
    \hat{u}_{\mathbf{n}}(t) = \langle u(\cdot, t), \phi_{\mathbf{n}} \rangle = \frac{1}{(2\pi)^{d/2}} \int_{\mathbb{T}^d} u(\mathbf{x}, t) e^{-i \mathbf{n} \cdot \mathbf{x}} \, \mathrm{d}\mathbf{x}.
\end{equation} Transforming \eqref{eq:wave_eq} into the frequency domain yields the evolution equation for each mode $\mathbf{n} \in \mathbb{Z}^d$:
\begin{equation} \label{eq:fourier_evolution}
    \frac{\mathrm{d}^2\hat{u}_{\mathbf{n}}}{\mathrm{d}t^2} + |\mathbf{n}|^2 \hat{u}_{\mathbf{n}}
    = -\epsilon \sum_{\mathbf{j} \neq \mathbf{n}} (\mathbf{n} \cdot \mathbf{j}) \hat{p}_{\mathbf{n}-\mathbf{j}} \hat{u}_{\mathbf{j}}(t) =: R_{\mathbf{n}}(t).
\end{equation}
The right-hand side $R_{\mathbf{n}}(t)$ represents the frequency coupling induced by the medium perturbation.

\subsection{Main Result}

We analyze the solution as a perturbation of the dominant mode. Let $u_{\mathbf{k}_0}^{(0)}(\mathbf{x}, t) = \cos(|{\mathbf{k}_0}| t )\phi_{\mathbf{k}_0}(\mathbf{x})$ be the solution in the unperturbed medium ($\epsilon=0$).

\begin{theorem}[Frequency Locality] \label{thm:main}
Consider the wave evolution over $t \in [0,T]$. For any initially unexcited frequency $\mathbf{n}_1 \neq \mathbf{k}_0$, we define the \textbf{first-order approximate solution} $\hat{u}_{\mathbf{n}_1}^{\text{approx}}(t)$ as the solution to the linearized inhomogeneous equation driven solely by the dominant mode:
\begin{equation} \label{eq:approx_ode}
    \frac{\mathrm{d}^2}{\mathrm{d}t^2} \hat{u}_{\mathbf{n}_1}^{\text{approx}} + |\mathbf{n}_1|^2 \hat{u}_{\mathbf{n}_1}^{\text{approx}} = -\epsilon (\mathbf{n}_1 \cdot \mathbf{k}_0) \hat{p}_{\mathbf{n}_1-\mathbf{k}_0} \hat{u}_{\mathbf{k}_0}^{(0)}(t),
\end{equation}
with zero initial conditions. By Duhamel's principle, this is explicitly given by:
\begin{equation}\label{eqn : u_approx}
    \hat{u}_{\mathbf{n}_1}^{\text{approx}}(t) := -\frac{\epsilon}{|{\mathbf{n}_1}|} ({\mathbf{n}_1} \cdot {\mathbf{k}_0}) \hat{p}_{\mathbf{n}_1-\mathbf{k}_0} \int_0^t \sin\left(|{\mathbf{n}_1}|(t-s)\right) \hat{u}_{\mathbf{k}_0}^{(0)}(s) \, \mathrm{d}s.
\end{equation}
If $\epsilon \ll 1$ and $\eta := \epsilon T |\mathbf{k}_0| \ll 1$, then the error is a higher-order small quantity:
\begin{equation}
    |\hat{u}_{\mathbf{n}_1}(t) - \hat{u}_{\mathbf{n}_1}^{\text{approx}}(t)| = o(\eta),
\end{equation}
while the magnitude of the approximation is $|\hat{u}_{\mathbf{n}_1}^{\text{approx}}(t)| \sim |\hat{p}_{\mathbf{n}_1-\mathbf{k}_0}| \eta$.
\end{theorem}

\subsection{Proof of Theorem~\ref{thm:main}}

The proof relies on energy estimates. We denote $v_t := \partial_t u$.

\begin{lemma}[Energy Conservation]\label{lem:energy}
The energy functional $E(t) = \frac{1}{2} \int_{\Omega} (|\partial_t u|^2 + c^2(\mathbf{x}) |\nabla u|^2) \mathrm{d}\mathbf{x}$ is conserved, i.e., $E(t) \equiv E(0)$.
\end{lemma}

\begin{lemma}[Uniform Bound on Spatial Gradient] \label{lem:grad_bound}
The weighted sum of squares of all frequency components is bounded for small $\epsilon$:
\begin{equation}\label{eq:grad_bound}
    \sum_{\mathbf{n} \in \mathbb{Z}^d} |\mathbf{n}|^2 |\hat{u}_{\mathbf{n}}(t)|^2 \le \frac{M}{m} |\mathbf{k}_0|^2 \le (1 + 3\epsilon) |\mathbf{k}_0|^2.
\end{equation}
\end{lemma}
\begin{proof}
By energy conservation, $m \|\nabla u(t)\|^2 \le 2E(t) = 2E(0) \le M \|\nabla u(0)\|^2 = M |\mathbf{k}_0|^2$. The result follows from Parseval's identity $\|\nabla u(t)\|^2 = \sum_{\mathbf{n}} |\mathbf{n}|^2 |\hat{u}_{\mathbf{n}}(t)|^2$ and $M/m = (1+\epsilon)/(1-\epsilon) = 1+2\epsilon + O(\epsilon^2)$.
\end{proof}

We decompose the dominant mode $\mathbf{k}_0$ into the principal term and a perturbation: $\hat{u}_{\mathbf{k}_0}(t) = \hat{u}_{\mathbf{k}_0}^{(0)}(t) + \delta \hat{u}_{\mathbf{k}_0}(t)$, where 
$\hat{u}_{\mathbf{\mathbf{k}_0}}^{(0)}$ satisfies
\begin{equation} \label{eq:fourier_evolution_unpreturbed}
    \frac{\mathrm{d}^2\hat{u}_{\mathbf{\mathbf{k}_0}}^{(0)}}{\mathrm{d}t^2} + |\mathbf{k}_0|^2 \hat{u}_{\mathbf{\mathbf{k}_0}}^{(0)}
    = 0,\, \hat{u}_{\mathbf{\mathbf{k}_0}}^{(0)}(0)=1,\,\frac{\mathrm{d}\hat{u}_{\mathbf{\mathbf{k}_0}}^{(0)}}{\mathrm{d}t}(0)=0 ,
\end{equation}
and $\delta \hat{u}_{\mathbf{k}_0}$ satisfies the forced equation \eqref{eq:fourier_evolution} with zero initial data.
\begin{equation} \label{eq:fourier_evolution_unpreturbed}
    \frac{\mathrm{d}^2\delta \hat{u}_{\mathbf{k}_0}}{\mathrm{d}t^2} + |\mathbf{k}_0|^2 \delta \hat{u}_{\mathbf{k}_0}
    =  -\epsilon \sum_{\mathbf{j} \neq \mathbf{k}_0} (\mathbf{k}_0 \cdot \mathbf{j}) \hat{p}_{\mathbf{k}_0-\mathbf{j}} \hat{u}_{\mathbf{j}},\, \delta \hat{u}_{\mathbf{k}_0}(0)=0,\,\frac{\mathrm{d}\delta \hat{u}_{\mathbf{k}_0}}{\mathrm{d}t}(0)=0 .
\end{equation}

\begin{lemma}[Estimate of Dominant Frequency Perturbation] \label{lem:k0_perturb}
Define $C_1=(1+3\epsilon)^{\frac{1}{2}}\|p\|_{L^2}$. Let $\eta := C_1 \epsilon T |\mathbf{k}_0|$. The perturbation $\delta \hat{u}_{\mathbf{k}_0}$ satisfies:
\begin{equation}
    |\delta\hat{u}_{\mathbf{k}_0}(t)| \le \eta, \quad |\delta\hat{v}_{\mathbf{k}_0}(t)| \le |\mathbf{k}_0| \eta.
\end{equation}
\end{lemma}
\begin{proof}
The source term in Equation~\eqref{eq:fourier_evolution} for $\mathbf{k}_0$ is $R_{\mathbf{k}_0}(t) = -\epsilon \sum_{\mathbf{j} \neq \mathbf{k}_0} (\mathbf{k}_0 \cdot \mathbf{j}) \hat{p}_{\mathbf{k}_0-\mathbf{j}} \hat{u}_{\mathbf{j}}$. By Cauchy-Schwarz and Lemma \ref{lem:grad_bound}:
\[
    |R_{\mathbf{k}_0}| \le \epsilon |\mathbf{k}_0| \|p\|_{L^2} \|\nabla u(t)\|_{L^2} \le C_1 \epsilon |\mathbf{k}_0|^2.
\]
Applying Duhamel's principle to $\delta \hat{u}_{\mathbf{k}_0}$ yields 
\begin{equation}
    \delta\hat{u}_{\mathbf{k}_0}(t) = \int_0^t \frac{\sin(|\mathbf{k}_0|(t-s))}{|\mathbf{k}_0|} R_{\mathbf{k}_0}(s) \, \mathrm{d}s.
\end{equation}
Thus $|\delta \hat{u}_{\mathbf{k}_0}(t)| \le \int_0^t |\mathbf{k}_0|^{-1} |R_{\mathbf{k}_0}(s)| \mathrm{d}s \le \eta$. The derivative $|\delta\hat{v}_{\mathbf{k}_0}(t)|$ follows similarly.
\end{proof}

\begin{lemma}[Energy Bound for Scattered Modes] \label{lem:scatt_energy}
The total energy of modes $\mathbf{n} \neq \mathbf{k}_0$ satisfies:
\begin{equation}
    E_{\text{scatt}}(t) := \sum_{\mathbf{n} \neq \mathbf{k}_0} \left( |\hat{v}_{\mathbf{n}}|^2 + (1-\epsilon)|\mathbf{n}|^2 |\hat{u}_{\mathbf{n}}|^2 \right) \le C \epsilon |\mathbf{k}_0|^2 + C \eta |\mathbf{k}_0|^2.
\end{equation}
\end{lemma}

\begin{proof}
Based on energy conservation and the lower bound $c^2(\mathbf{x}) \ge 1-\epsilon$, the total energy satisfies the inequality:

\begin{equation} \label{eq:energy_split}
    \left( |\hat{v}_{\mathbf{k}_0}|^2 + (1-\epsilon)|\mathbf{k}_0|^2 |\hat{u}_{\mathbf{k}_0}|^2 \right) + \sum_{\mathbf{n} \neq \mathbf{k}_0} \left( |\hat{v}_{\mathbf{n}}|^2 + (1-\epsilon)|\mathbf{n}|^2 |\hat{u}_{\mathbf{n}}|^2 \right) \le 2E(0).
\end{equation}
Substituting the decomposition $\hat{u}_{\mathbf{k}_0} = \hat{u}_{\mathbf{k}_0}^{(0)} + \delta \hat{u}_{\mathbf{k}_0}$ and $\hat{v}_{\mathbf{k}_0} = \hat{v}_{\mathbf{k}_0}^{(0)} + \delta \hat{v}_{\mathbf{k}_0}$ into the first term (denoted as $E_{\mathbf{k}_0}(t)$), we expand the squares:
\begin{equation}
E_{\mathbf{k}_0}(t) =
|\hat{v}_{\mathbf{k}_0}^{(0)}|^2
+ 2\text{Re}\!\left(\hat{v}_{\mathbf{k}_0}^{(0)} \overline{\delta \hat{v}_{\mathbf{k}_0}}\right)
+ |\delta \hat{v}_{\mathbf{k}_0}|^2
+ (1-\epsilon)|\mathbf{k}_0|^2
\left(
|\hat{u}_{\mathbf{k}_0}^{(0)}|^2
+ 2\text{Re}\!\left(\hat{u}_{\mathbf{k}_0}^{(0)} \overline{\delta \hat{u}_{\mathbf{k}_0}}\right)
+ |\delta \hat{u}_{\mathbf{k}_0}|^2
\right).
\end{equation}
Recall that the initial energy is $2E(0) =  |\mathbf{k}_0|^2 |\hat{u}_{\mathbf{k}_0}^{(0)}|^2$. Rearranging \eqref{eq:energy_split} yields:
\begin{align*}
    E_{\text{scatt}}(t) &\le 2E(0) - E_{\mathbf{k}_0}(t) \\
    &= \epsilon |\mathbf{k}_0|^2 |\hat{u}_{\mathbf{k}_0}|^2 - \left( 2\text{Re}(\hat{v}_{\mathbf{k}_0}^{(0)} \overline{\delta \hat v_{\mathbf{k}_0}}) + 2(1-\epsilon)|\mathbf{k}_0|^2 \text{Re}(\hat{u}_{\mathbf{k}_0}^{(0)} \overline{\delta \hat u_{\mathbf{k}_0}}) \right) - \left( |\delta \hat{v}_{\mathbf{k}_0}|^2 + (1-\epsilon)|\mathbf{k}_0|^2 |\delta \hat{u}_{\mathbf{k}_0}|^2 \right).
\end{align*}
Using the bounds from Lemma \ref{lem:k0_perturb} ($|\delta \hat{u}_{\mathbf{k}_0}| \le \eta$, $|\delta \hat{v}_{\mathbf{k}_0}| \le |\mathbf{k}_0|\eta$) and the bound $|\hat{u}_{\mathbf{k}_0}| \le 1 + 3\epsilon$, we estimate the terms on the right-hand side. The term  $\epsilon |\mathbf{k}_0|^2 |\hat{u}_{\mathbf{k}_0}|^2 $ is bounded by $C \epsilon |\mathbf{k}_0|^2$. The cross terms (linear in $\delta u, \delta v$) are bounded by $C |\mathbf{k}_0|^2 \eta$. The quadratic terms are of order $O(\eta^2)$, which is dominated by $O(\eta)$ since $\eta \ll 1$. Thus, $E_{\text{scatt}}(t) \le C(\epsilon + \eta) |\mathbf{k}_0|^2$.
\end{proof}

\begin{proof}[Proof of Theorem~\ref{thm:main}]
By Duhamel's principle, the true solution $\hat{u}_{\mathbf{n}_1}(t)$ satisfies
\[
    \hat{u}_{\mathbf{n}_1}(t) = -\frac{\epsilon}{|\mathbf{n}_1|} \int_0^t \sin(|\mathbf{n}_1|(t-s)) \left( \sum_{\mathbf{j} \neq \mathbf{n}_1} (\mathbf{n}_1 \cdot \mathbf{j}) \hat{p}_{\mathbf{n}_1-\mathbf{j}} \hat{u}_{\mathbf{j}}(s) \right) \mathrm{d}s.
\]
The error $\Delta(t) := |\hat{u}_{\mathbf{n}_1}(t) - \hat{u}_{\mathbf{n}_1}^{\text{approx}}(t)|$ arises from two sources:
\begin{align*}
    \Delta(t) \le \frac{\epsilon}{|\mathbf{n}_1|} \int_0^t \Bigg| \underbrace{\sum_{\mathbf{j} \neq \mathbf{k}_0, \mathbf{n}_1} (\mathbf{n}_1 \cdot \mathbf{j}) \hat{p}_{\mathbf{n}_1-\mathbf{j}} \hat{u}_{\mathbf{j}}(s)}_{\text{Part I: Scattering Interactions}} + \underbrace{(\mathbf{n}_1 \cdot \mathbf{k}_0) \hat{p}_{\mathbf{n}_1-\mathbf{k}_0} \delta \hat{u}_{\mathbf{k}_0}(s)}_{\text{Part II: Perturbation of } \mathbf{k}_0} \Bigg| \, \mathrm{d}s.
\end{align*} 

We claim that there holds that $\Delta(t) = O(\eta^{1.5}) + O(\eta^2) = o(\eta)$.
To estimate  Part I, we use Cauchy-Schwarz and Lemma \ref{lem:scatt_energy} and obtain :
\[
    \text{Part I} \le |\mathbf{n}_1| \|p\|_{L^2} \left( \sum_{\mathbf{j} \neq \mathbf{k}_0} |\mathbf{j}|^2 |\hat{u}_{\mathbf{j}}|^2 \right)^{1/2} \lesssim |\mathbf{n}_1| (\epsilon + \eta)^{1/2} |\mathbf{k}_0|.
\]
Integrating over $[0, t]$ gives a contribution of order $\epsilon T |\mathbf{k}_0| (\epsilon+\eta)^{1/2} \sim \eta^{1.5} = o(\eta)$.
To estimate Part II, we apply Lemma \ref{lem:k0_perturb} ($|\delta \hat{u}_{\mathbf{k}_0}| \le \eta$) and obtain :
\[
    \text{Part II} \le |\mathbf{n}_1| |\mathbf{k}_0| \cdot 1 \cdot \eta.
\]
Integrating over $[0, t]$ then gives a contribution of order $\epsilon T |\mathbf{k}_0| \eta \sim \eta^2 = o(\eta)$.

To verify the magnitude of the approximation $\hat{u}_{\mathbf{n}_1}^{\text{approx}}(t)$, we substitute $\hat{u}_{\mathbf{k}_0}^{(0)}(s) = \cos(|\mathbf{k}_0|s)$ into the explicit formula. Using the identity $2\sin A \cos B = \sin(A+B) + \sin(A-B)$, the integral \eqref{eqn : u_approx}  can be decomposed as:
\begin{equation}\label{eqn : u_hat_approx}
    \hat{u}_{\mathbf{n}_1}^{\text{approx}}(t) = -\frac{\epsilon (\mathbf{n}_1 \cdot \mathbf{k}_0)}{2|\mathbf{n}_1|} \hat{p}_{\mathbf{n}_1-\mathbf{k}_0} \int_0^t \Big[ \sin\left(|\mathbf{n}_1|t - (|\mathbf{n}_1|-|\mathbf{k}_0|)s\right)
    + \sin\left(|\mathbf{n}_1|t - (|\mathbf{n}_1|+|\mathbf{k}_0|)s\right) \Big] \, \mathrm{d}s.
\end{equation}
Note that $\frac{\mathbf{n}_1 \cdot \mathbf{k}_0}{|\mathbf{n}_1|} \sim O(|\mathbf{k}_0|)$ and the integral term in \eqref{eqn : u_hat_approx} is comparable with $t$ when $|\mathbf{n}_1| \approx |\mathbf{k}_0|$. Consequently, the total magnitude of \eqref{eqn : u_hat_approx} scales as $\epsilon |\mathbf{k}_0| T = \eta$. Since the error $\Delta(t) = o(\eta)$, the approximate solution $\hat{u}_{\mathbf{n}_1}^{\text{approx}}$ captures the dominant dynamics.
\end{proof}

 In summary, Theorem \ref{thm:main} elucidates the mechanism of energy transfer from the initial frequency $\mathbf{k}_0$ to scattered modes $\mathbf{n}_1$. The explicit form of the approximate solution demonstrates that the amplitude of the scattered wave is directly proportional to the Fourier coefficient $\hat{p}_{\mathbf{n}_1-\mathbf{k}_0}$. This dependence establishes the principle of \textit{frequency locality}: given the smoothness of the medium perturbation $p(\mathbf{x})$, its Fourier coefficients decay rapidly as the spectral distance $|\mathbf{n}_1 - \mathbf{k}_0|$ increases. Consequently, significant energy transfer is effectively confined to a narrow spectral neighborhood around the dominant frequency $\mathbf{k}_0$, rendering long-range interactions in the frequency domain negligible.

\section{The windowed Fourier propagator}\label{sec:wfp_methodology}

 Inspired by Theorem~\ref{thm:main}, we aim to design a neural network with the following features: 
First,Due to the principle of superposition, we want to design a neural network to stimulate the evolution of each Fourier component. In particular, the network is layered by time, and the neurons model Fourier modes.
Second,The initial values and the medium parameter are set to be the network inputs.
Third, Given a numerical threshold, one may set a frequency neighbourhood, called the frequency window, near the driving frequencies in the frequency domain. The frequencies beyond the window are all negligible due to the decay of the Fourier coefficients of the medium parameter.

The windowed Fourier propagator learns a localized evolution operator in the frequency domain  through a multi-stage workflow:
\begin{enumerate}
	\item \textbf{Input extraction}: The inputs, including the initial value and the medium parameter, are transformed into the frequency domain. In the meanwhile, their initial frequencies and amplitudes are extracted and stored as vectors of finite dimension.
	
	\item \textbf{Localized evolution}: The neural network maps each vector to a propagator localized in the frequency window with a given threshold. The latter is a small learnable dictionary describing how the amplitude of each driving mode spreads to other Fourier modes. 
	
	\item \textbf{Amplitude scaling and aggregation}: Each propagator is scaled by its original amplitude, and the complete set of evolved Fourier coefficients is synthesized.
	
	\item \textbf{Output Reconstruction}: The inverse (discrete) Fourier transform converts the synthesized coefficients back into the full-field spatial solution. 
\end{enumerate}

In this section, we elucidate this process for the Cauchy problem \eqref{eq:wave_equation_homogeneous} with an initial value $f$ and a medium structural function $c^2$ dominated by low frequencies.

\subsection{The frequency window and extraction function}\label{def:significant_frequency_extraction}

 The WFP is a neural operator in the frequency domain. To attain low computational complexity, it only processes the data through some extraction functions within some frequency windows. We first formulate these concepts.

The initial value $f$ can be mapped, via the discrete Fourier transform (DFT), to  frequency-amplitude pairs $\{(\mathbf{k}_i, \hat{f}_{\mathbf{k}_i})\}_{i=1}^M$. Specifying a numerical threshold for $\hat{f}$, say $\tau > 0$, 
we capture the frequency-amplitude pairs whose amplitudes exceed $\tau$, and filter out other components. Such pairs are collected in the set \[\mathcal{F}_\tau := \left\{ (\mathbf{k}_i, \hat{f}_{\mathbf{k}_i}) \in \mathbb{Z}^d \times \mathbb{C} : |\hat{f}_{\mathbf{k}_i}| \geq \tau \right\}.\]
For each element in $\mathcal{F}_\tau$, the frequency $\mathbf{k}_i$ and the  amplitude $\hat{f}_{\mathbf{k}_i}$ will be stored. This procedure yields a concise $[N, d+2]$ tensor, with $N := |\mathcal{F}_\tau|$, representing the initial condition as a collection of $[\,\mathrm{Re}(\hat{f}_{\mathbf{k}_i}),\, \mathrm{Im}(\hat{f}_{\mathbf{k}_i}),\, \mathbf{k}_i\,]$, enabling efficient parallel processing.

To extract the data of the structural function $c^2$, we introduce the frequency window and the windowed frequency extraction function. 
A frequency window is a ball in the frequency domain $\mathbb{Z}^d$, and a windowed frequency extraction function picks up the Fourier coefficients of a function  with in a frequency window. More precisely, 
\begin{definition}\label{def:window_extraction}
	 Let $\mathbf{k}_c \in \mathbb{Z}^d$ and $r > 0$. The frequency window, centred at $\mathbf{k}_c$ with radius $r$, is taken to be  \[\mathcal{W}(\mathbf{k}_c, r) := \left\{ \mathbf{k} \in \mathbb{Z}^d : \|\mathbf{k} - \mathbf{k}_c\|_{\infty} \leq r \right\}.\]
	For $h \in C^{l}(\Omega)$, the windowed frequency extraction function of $h$ is defined to 
	 \begin{align}\mathcal{E}(h, \mathbf{k}_c, r) := \left\{ \hat{h}_{\mathbf{k}} : \mathbf{k} \in \mathcal{W}(\mathbf{k}_c, r) \right\}.\label{eq : window extraction}\end{align}
	 These extracted coefficients are flattened into a real, one-dimensional array of length $2 \times (2r + 1)^d$.
\end{definition}

We assume $c^2(x)$ is a smooth function dominated by its low frequency part. Hence, 
\begin{equation}\label{velocity_feature}
	\mathcal{C} := \mathcal{E}\left(c^2, \mathbf{0}, r_c\right)
\end{equation} 
is a good approximation of $c^2$ for a moderate $r_c > 0$. This vector $\mathcal{C}$ serves as a compact representation of the medium parameter, independent of the spatial resolution.

   For a driving frequency $\mathbf{k}_0$, the target output vector for  \eqref{eq:wave_equation_homogeneous}  is set to be the extraction of $u(\mathbf{x}, t_{\text{final}})$ within the window centred at $\mathbf{k}_0$ with some radius $r_w > 0$. That is
\begin{equation}\label{window}
	\mathcal{F}(\mathbf{k}_0, r_w) := \mathcal{E}\left(u(\mathbf{x}, t_{\text{final}}), \mathbf{k}_0, r_w\right).
\end{equation}


\subsection{Architecture of the WFP}

The WFP operates through a structured pipeline to predict wave evolution, designed for $d$-dimensional problems and ensuring linearity with respect to initial conditions:

\begin{figure}[ht]
	\begin{center}
		\includegraphics[,width=0.8\linewidth]{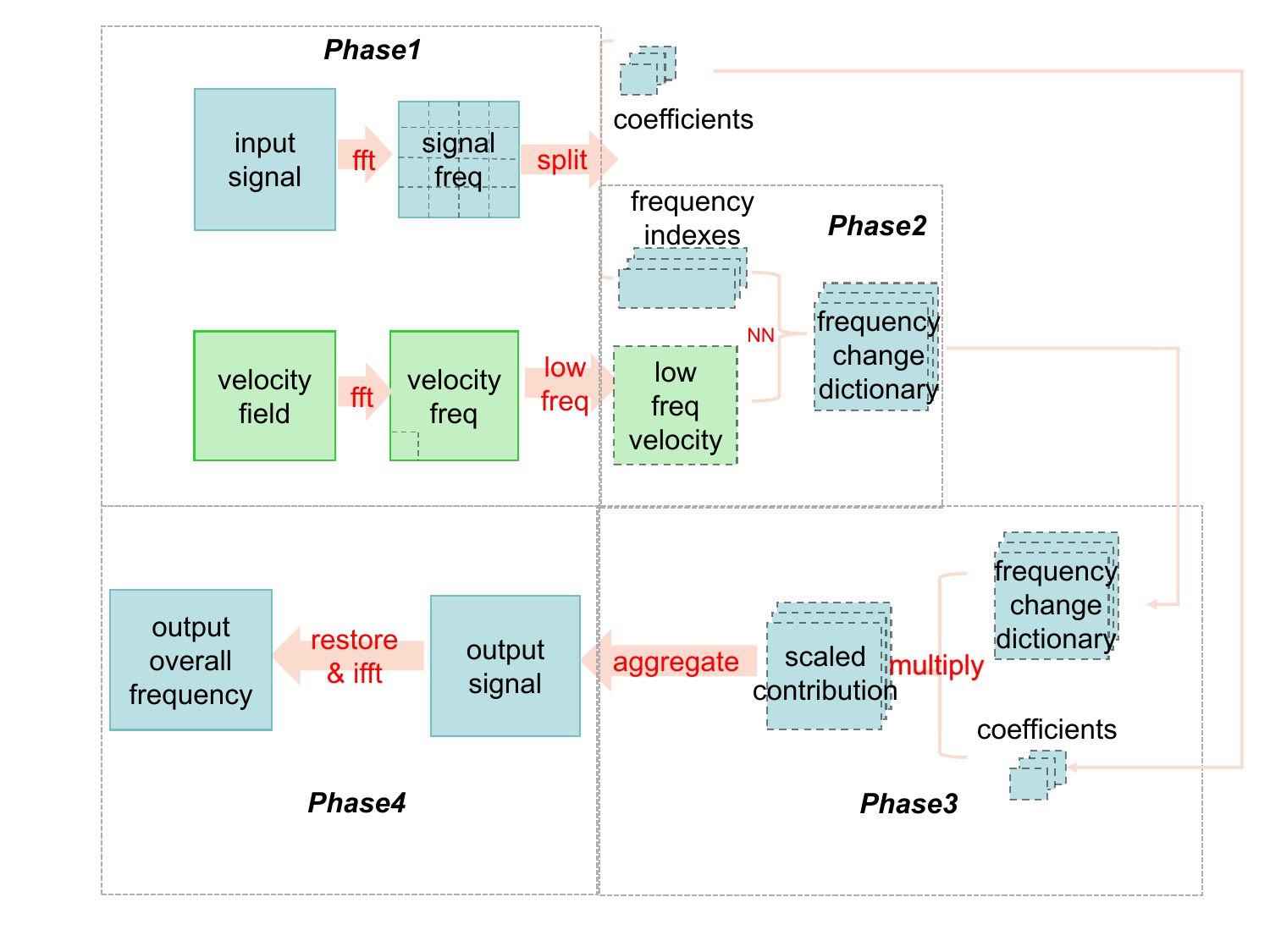}
	\end{center}
	\caption{Overview of the WFP network architecture} 
	\label{fig:architecture}
\end{figure}\begin{itemize}
\item  \textbf{Phase 1: Input Extraction} 

The initial condition $u(\mathbf{x},0)$ (and $u_t(\mathbf{x},0)$, if applicable, 
processed similarly) is transformed into the Fourier coefficient  $\hat{u}_{\mathbf{k}}(0)$ over the frequency domain $\mathbb{Z}^d$ using 
a $d$-dimensional Discrete Fourier Transform (DFT). 
The driving components $(\mathbf{k}_i, \hat{u}_{\mathbf{k}_i}(0))$ 
are extracted by selecting frequency vectors $\mathbf{k}_i \in \mathbb{Z}^d$ whose 
complex amplitudes $\hat{u}_{\mathbf{k}_i}(0)$ exceed a predefined threshold $\tau$, 
as defined in Definition~\ref{def:significant_frequency_extraction}. 
Simultaneously, the spatially varying velocity field $c(\mathbf{x})$ is transformed via DFT, 
and its low-frequency Fourier coefficients $\hat{c}_{\mathbf{k}}$ with $\|\mathbf{k}\|_{\infty} < r_c$
are extracted and organized into a compact vector $\mathcal{C}$ 
in \eqref{velocity_feature}, which serves as a resolution-independent 
representation of the medium.

\item \textbf{Phase 2:  Localized evolution} 

For each driving frequency component $(\mathbf{k}_i, \hat{u}_{\mathbf{k}_i}(0))$, an input token is formed by concatenating the frequency vector $\mathbf{k}_i$ with the velocity field representation $\mathcal{C}$. The neural network processes this token to predict the short-term ($\Delta t$) evolution  of mode $\mathbf{k}_i$ under $\mathcal{C}$, outputting an evolution dictionary $\mathcal{D}(\mathbf{k}_i, \mathcal{C})$ containing $(2r_w + 1)^d$ complex amplitudes (where $r_w$ is the window radius in \eqref{window}).  
This dictionary describes the predicted amplitudes of the waves at frequency
within 
$\mathcal{W}(\mathbf{k}_i, r_w)$. 

\item \textbf{Phase 3: Amplitude Scaling and Aggregation}
 
Each evolution dictionary $\mathcal{D}(\mathbf{k}_i, c)$ predicted by the network represents the evolution pattern for the Fourier mode  $e^{i \mathbf{k}_i} \cdot x$. For general initial values, this dictionary is scaled by the Fourier coefficient $\hat{u}_{\mathbf{k}_i}(0)$, yielding the scaled contribution \[\hat{U}_{\text{contrib}}(\mathbf{k}_i) = \hat{u}_{\mathbf{k}_i}(0) \cdot \mathcal{D}(\mathbf{k}_i, c).\] In this process, the linearity with respect to initial conditions is preserved throughout the evolution. Subsequently, the amplitude $\hat{u}_{\mathbf{k}'}(\Delta t)$ for each frequency $\mathbf{k}'$   is computed by aggregating all scaled contributions from initial modes whose windows contain that frequency. That is,  \[\hat{u}_{\mathbf{k}'}(\Delta t) = \sum_{(\mathbf{k}_i, \hat{u}_{\mathbf{k}_i}(0)) \text{ s.t. } \mathbf{k}' \in \mathcal{W}(\mathbf{k}_i, r)} \left( \hat{U}_{\text{contrib}}(\mathbf{k}_i) \right)_{\mathbf{k}'}.\] 

\item \textbf{Phase 4: Output Reconstruction} 

An inverse $d$-dimensional DFT is applied to $\hat{u}_{\mathbf{k}'}(\Delta t)$ to reconstruct the wave $u(\mathbf{x}, \Delta t)$ in the spatial domain.
\end{itemize}

We remark that this architecture ensures linearity with respect to initial conditions.
This is due to that the network's prediction $D(\mathbf{k}_i, c)$ is independent of the initial 
amplitude $\hat{u}_{\mathbf{k}_i}(0)$ and the initial amplitude is applied linearly during the scaling step.


	

\subsection{Neural Network Architecture for Learning the Propagator}
Learning the mapping from an initial mode and medium properties to an evolved window is challenging. Theorem~\ref{thm:main} shows the the evolution is highly structured and sparse, and the underlying solution is oscillatory, properties that are difficult for standard networks to capture efficiently.

To incorporate the local energy transfer demonstrated in Theorem~\ref{thm:main}, we employ a gated architecture (Figure~\ref{fig:gate_structure}), inspired by its use in recurrent networks \cite{cho2014learning,hochreiter1997long}. The network splits the input token into two branches: a main branch that predicts the evolution dictionary, and a gate branch with a sigmoid activation that produces a modulation vector. The final output is the element-wise product of these branches. This mechanism allows the network to learn a data-driven filter, enforcing a soft sparsity that helps focus on the most relevant physical interactions and improves accuracy. The detailed structure is in Appendix~\ref{appendix:nn_operations}.

The choice of activation function is another factor in performance. While standard functions like $\mathrm{ReLU}$ or $\mathrm{Tanh}$ are universal approximators, they may not catch the oscillatory nature of the wave equation effectively. To better align the network's inductive bias with the oscillatory behavior of the wave equation, we introduce a custom activation function that explicitly incorporates oscillatory behavior: $\sigma_{custom}(x) = A \exp(-Bx^2) \sin(Cx)$. This form heuristically resembles canonical wave solutions (e.g., wave packets) and, as demonstrated in Section~\ref{ablation_study}, leads to significantly improved performance compared to standard counterparts.

\begin{figure}[h]
	\centering
        \includegraphics[width=0.3\textwidth]{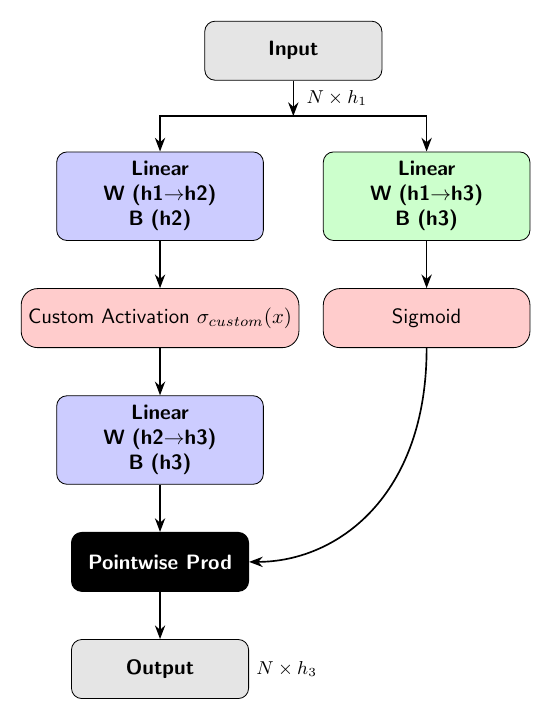}
	\caption{The structure of the neural network. An input token is processed through two branches: a main processing path with a custom activation and a gate branch with a sigmoid. Their element-wise product yields the final output, allowing the gate to control information flow \cite{hochreiter1997long} and emphasize relevant features.}
	\label{fig:gate_structure}
\end{figure}

\section{Numerical Experiments}\label{sec:experiments}
In this section, we present a comprehensive empirical evaluation of the Windowed Fourier Propagator (WFP) to validate its accuracy and generalization capabilities. We begin by detailing the data generation process and the construction of both in-distribution and out-of-distribution datasets in Section~\ref{sec:dataset_construction}. Subsequently, Section~\ref{sec:id_performance} demonstrates the model's fitting performance on 2D wave propagation tasks within the training distribution, including complex Gaussian beam evolutions. To assess the robustness of our framework, Section~\ref{sec:ood_generalization} investigates the model's generalization to out-of-distribution scenarios involving discontinuous media and stronger perturbations than those seen during training. Finally, we conclude with an ablation study in Section~\ref{ablation_study} to analyze the impact of activation function selection on the network's learning dynamics.

\subsection{Dataset Construction}\label{sec:dataset_construction}
The WFP is trained and evaluated on datasets of input-output pairs generated by numerically solving the 2D wave equation \eqref{eq:wave_equation_homogeneous}. Each sample in the dataset consists of a unique, smoothly varying wave speed field $c(\mathbf{x})$, a single-frequency initial condition $u(\mathbf{x},0)$, and the corresponding evolved wave fields $u(\mathbf{x}, T_{\text{final}})$ and $u_t(\mathbf{x}, T_{\text{final}})$ which serve as the target.

\textbf{In-Distribution Training Data:}
The training data is designed to represent wave propagation in smoothly varying inhomogeneous media.

\textbf{1. Wave Speed Field $c(\mathbf{x})$:} The velocity fields for the in-distribution dataset are generated by adding random perturbations to a baseline speed of 1.0. For a spatial coordinate $\mathbf{x}=(x,y) \in^2$, the field is constructed as:
\[ 
c(x,y) = (1 + C_{\text{offset}}) + \sum_{k=1}^{12} A_k \cdot \left[ p_k\cos(2\pi \mathbf{f}_k \cdot \mathbf{x}) + (1-p_k)\sin(2\pi \mathbf{f}_k \cdot \mathbf{x}) \right]
\]
where $C_{\text{offset}}$ is a uniform random offset in $[-0.02, 0.02]$.
 $A_k = \alpha \cdot 0.9^k$ is the amplitude of the $k$-th perturbation, with a base disturbance strength $\alpha=0.03$. The term $0.9^k$ ensures that higher-frequency components have smaller amplitudes, resulting in a smooth field.
  $p_k \sim \text{Bernoulli}(0.5)$ is a random variable that selects between sine and cosine.
 $\mathbf{f}_k = (f_{x,k}, f_{y,k})$ is the 2D frequency vector for the $k$-th perturbation. The frequencies are randomly sampled integers, with $f_{x,k} \in [1, \lfloor k/2 \rfloor + 2]$ and $f_{y,k} \in [-\lfloor k/2 \rfloor - 2, -1] \cup [1, \lfloor k/2 \rfloor + 2]$. This sampling strategy ensures a rich variety of smoothly varying media.

\textbf{2. Initial Condition and Ground Truth:} The initial condition for each sample is a single-frequency plane wave $u(\mathbf{x},0) = \sin(2 \pi \mathbf{f}_{\text{in}} \cdot \mathbf{x})$, where the input frequency $\mathbf{f}_{\text{in}}=(f_x, f_y)$ is randomly selected with integer components $f_x\in[16,96]$ and $f_y\in[-96,-16]\cup[16,96]$. The ground-truth evolved state at $T_{\text{final}} = 0.02\ \mathrm{s}$ is computed by solving the wave equation using a high-fidelity, second-order finite-difference scheme on a fine $2049\times2049$ grid.

\textbf{Out-of-Distribution (OOD) Test Data:}
To evaluate the model's generalization capabilities, we construct two types of OOD datasets where the velocity fields deviate from the training distribution.

\textbf{1. Discontinuous Media:} This dataset assesses the model's performance on non-smooth media. A sharp, localized discontinuity is introduced into an otherwise smooth background field. The velocity field $c(x,y)$ is defined as:
\[
c(x,y) = 0.99 + 0.01\sin(2\pi(10x+10y)) + \delta \cdot \mathbf{1}_{S}(x,y)
\]
where $\mathbf{1}_{S}(x,y)$ is an indicator function for a square region $S$. Specifically, $S$ is defined as the square region $[0.28, 0.48] \times [0.28, 0.48]$ in the unit domain. The parameter $\delta$ controls the strength of the discontinuity, which is varied from $0.005$ to $0.2$ in our experiments.

\textbf{2. Stronger Perturbations:} This dataset tests the model's ability to extrapolate to media with larger velocity variations than seen during training. The velocity fields are generated using the same formula as the in-distribution data, but the base disturbance strength $\alpha$ is systematically increased up to $0.2$, exceeding the training value of $\alpha=0.03$.


	
	

\subsection{Wave Equation - In Distribution Generalization}\label{sec:id_performance}

For the 2D cases in this section, we present the fitting results 
for the evolution of both plane waves and Gaussian wave packets, 
which serve as initial conditions with complex frequency components.
The velocity fields used in these experiments are sampled from the same distribution as the training data (i.e., in-distribution).

\textbf{2D Plane Waves Visualization (Figure~\ref{fig:plane_wave_performance_id}).}
The initial condition is set as $\sin(2\pi (80x+64y))+\cos(2\pi (64x+80y))$ and velocity field is randomly generated within the distribution of training data. Satisfactory fitting performance is achieved. 
\begin{figure}[h]
	\centering
	\begin{subfigure}[b]{0.24\textwidth}
		\centering
		\includegraphics[width=\textwidth]{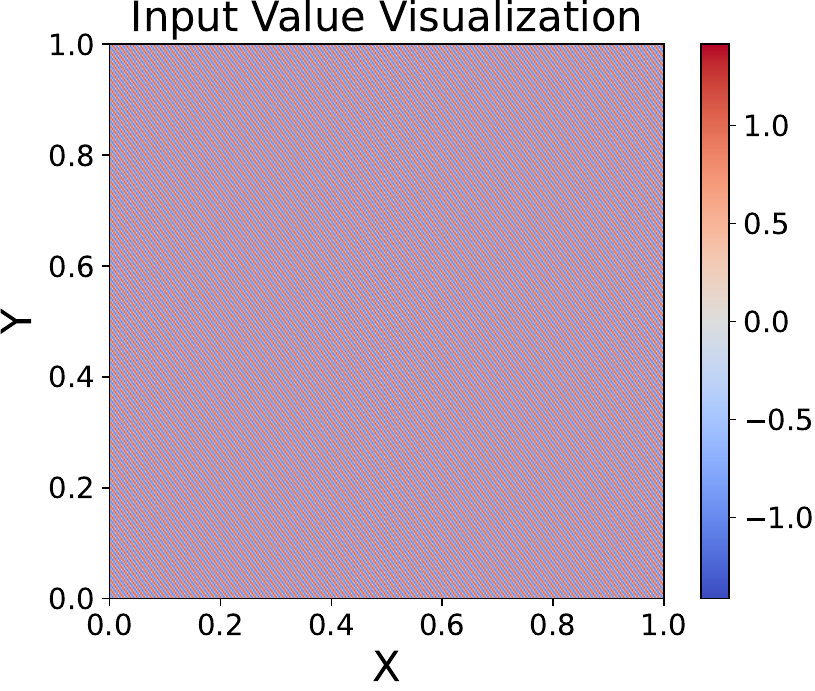}
		\caption{Initial}
		\label{fig:initial_condition}
	\end{subfigure}
	\hfill
	\begin{subfigure}[b]{0.24\textwidth}
		\centering
		\includegraphics[width=\textwidth]{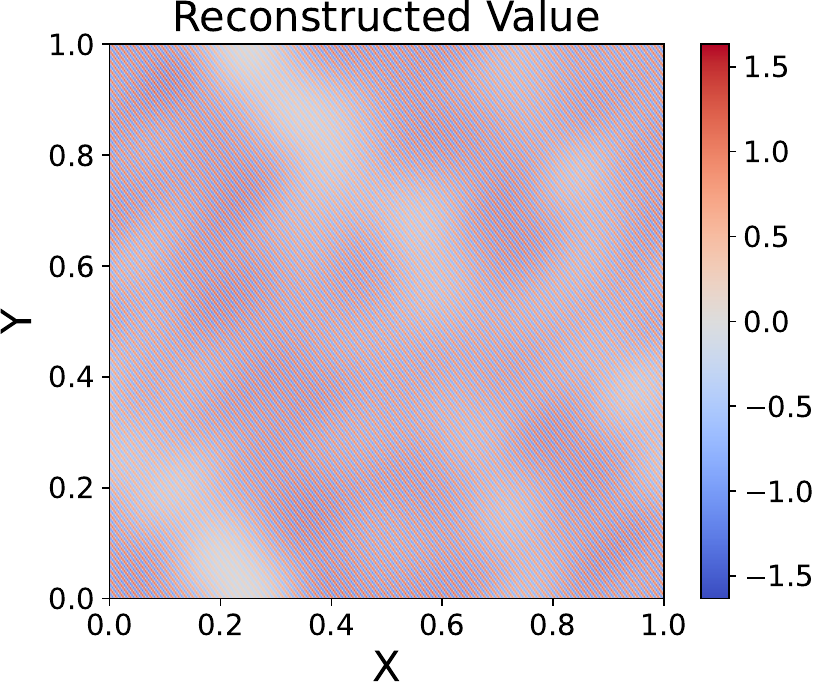}
		\caption{Ref.}
		\label{fig:2d_original}
	\end{subfigure}
	\hfill
	\begin{subfigure}[b]{0.24\textwidth}
		\centering
		\includegraphics[width=\textwidth]{figures/disturbance_0.05_reconstructed_value.pdf}
		\caption{Recon.}
		\label{fig:2d_reconstructed}
	\end{subfigure}
	\hfill
	\begin{subfigure}[b]{0.24\textwidth}
		\centering
		\includegraphics[width=\textwidth]{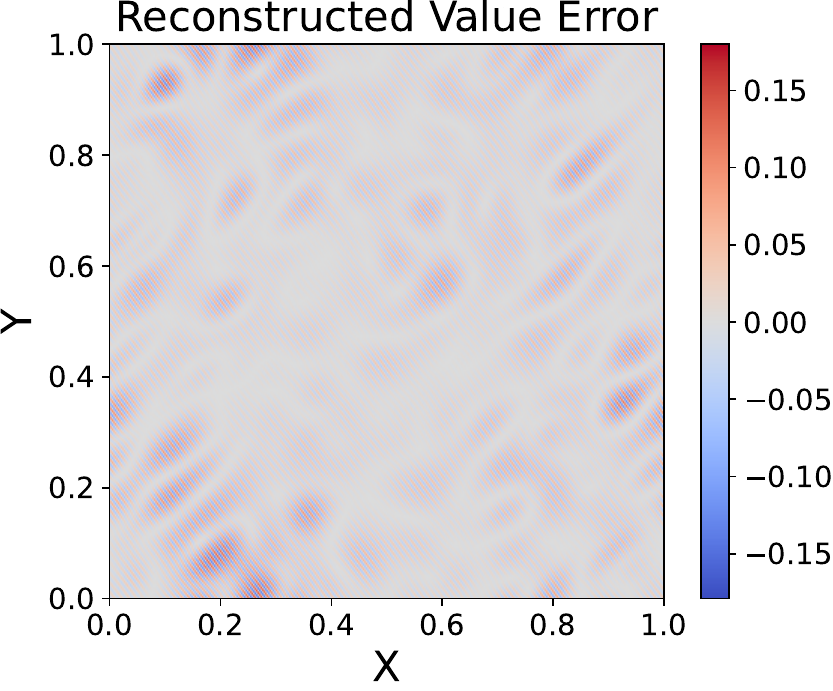}
		\caption{Error}
		\label{fig:2d_error}
	\end{subfigure}
	
	\vspace{1em} 

	\raggedright
	\textbf{Visualization of 2D Network Performance - Plane Waves:}
	Comparison between original and reconstructed solutions for 
	overall view.
	\caption{Network Performance for 2D In Distribution Case - Limited Initial Frequencies}
	\label{fig:plane_wave_performance_id}
\end{figure}

\textbf{Gaussian Beam Fitting Visualization (Figure~\ref{fig:gaussian_fit_performance}).}
A Gaussian wave packet 
is characterized by the product of a trigonometric function and a negative exponential function, typically exhibiting pronounced variations within a localized region in both physical and frequency domains. \textbf{Although the neural network is trained exclusively on single-frequency components, 
	the preservation of linearity within 
	the architecture ensures robust generalization 
	to arbitrary initial conditions.}

\begin{figure}[h]
	\centering
	\begin{subfigure}[b]{0.24\textwidth}
		\centering
		\includegraphics[width=\textwidth]{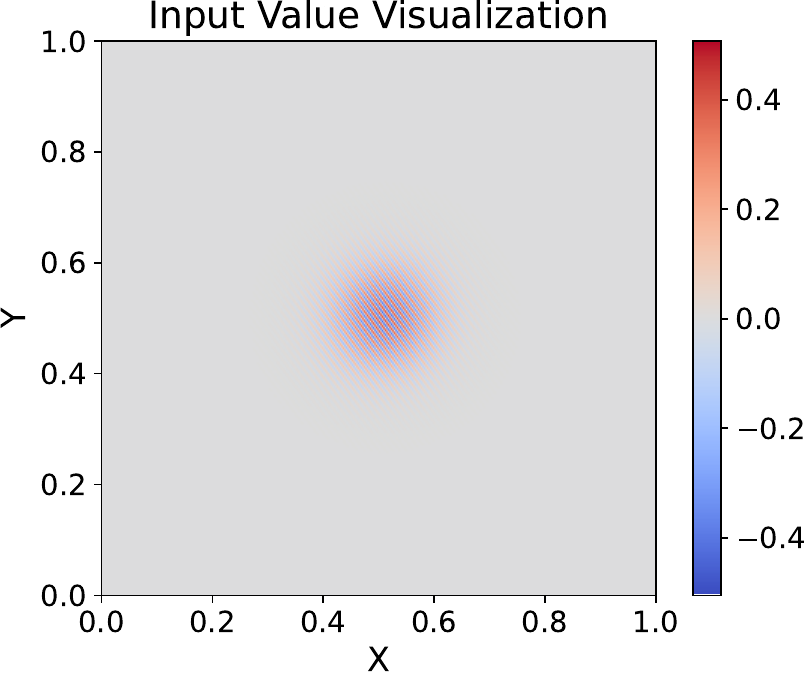}
		\caption{Initial}
		\label{fig:gaussian_initial_condition}
	\end{subfigure}
	\hfill
	\begin{subfigure}[b]{0.24\textwidth}
		\centering
		\includegraphics[width=\textwidth]{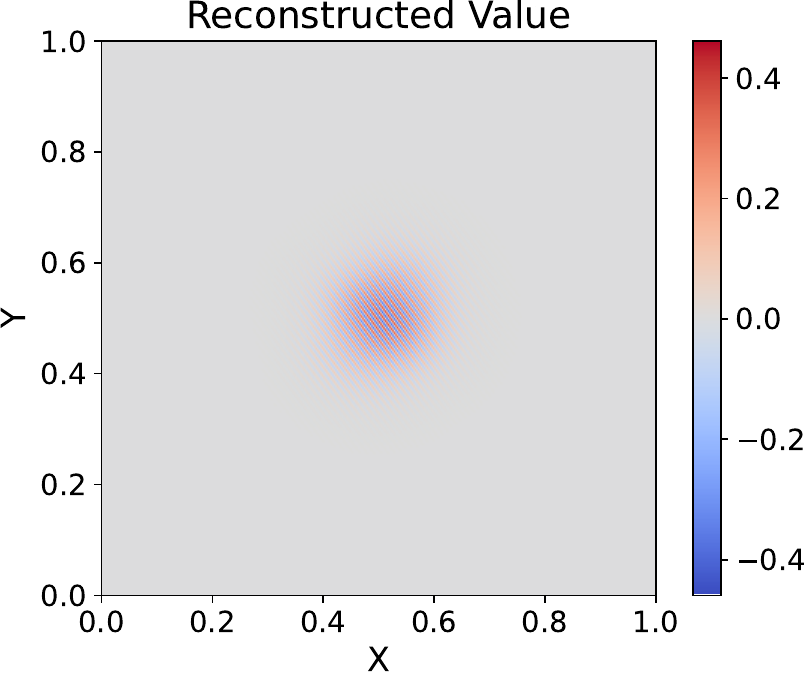}
		\caption{Ref.}
		\label{fig:gaussian_2d_original}
	\end{subfigure}
	\hfill
	\begin{subfigure}[b]{0.24\textwidth}
		\centering
		\includegraphics[width=\textwidth]{figures/gaussian_disturbance_0.04_reconstructed_value.pdf}
		\caption{Recon.}
		\label{fig:gaussian_2d_reconstructed}
	\end{subfigure}
	\hfill
	\begin{subfigure}[b]{0.24\textwidth}
		\centering
		\includegraphics[width=\textwidth]{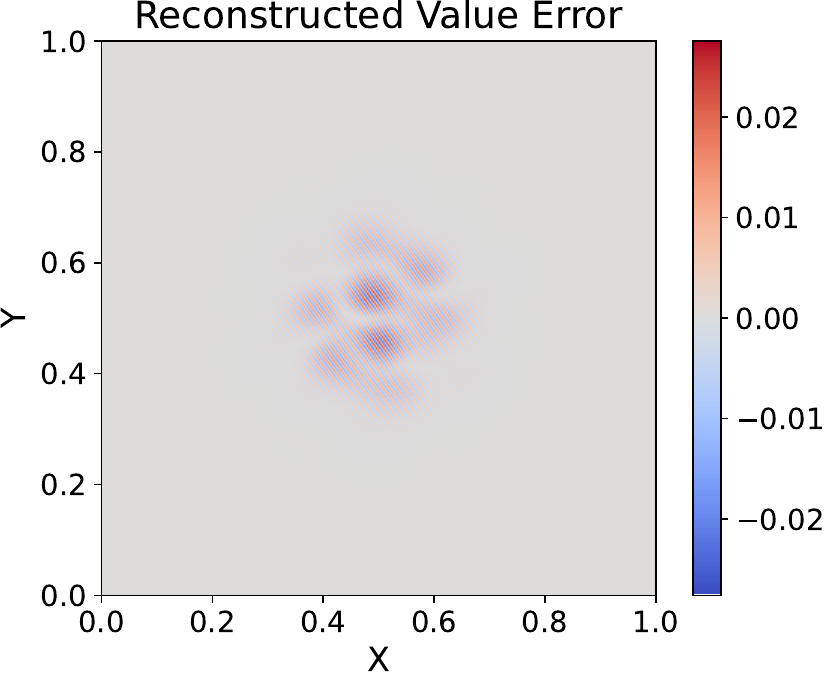}
		\caption{Error}
		\label{fig:gaussian_2d_error}
	\end{subfigure}
	
	\vspace{1em}

	\raggedright
	\textbf{Visualization of 2D Network Performance - Gaussian Beam:}
	Network can be used to handle complex initial conditions.
	\caption{Network Performance for 2D In Distribution Case - Gaussian Beam}
	\label{fig:gaussian_fit_performance}
\end{figure}
\begin{figure}[!h]
\centering
\begin{subfigure}[h]{0.24\textwidth}
\centering
\includegraphics[width=\textwidth]{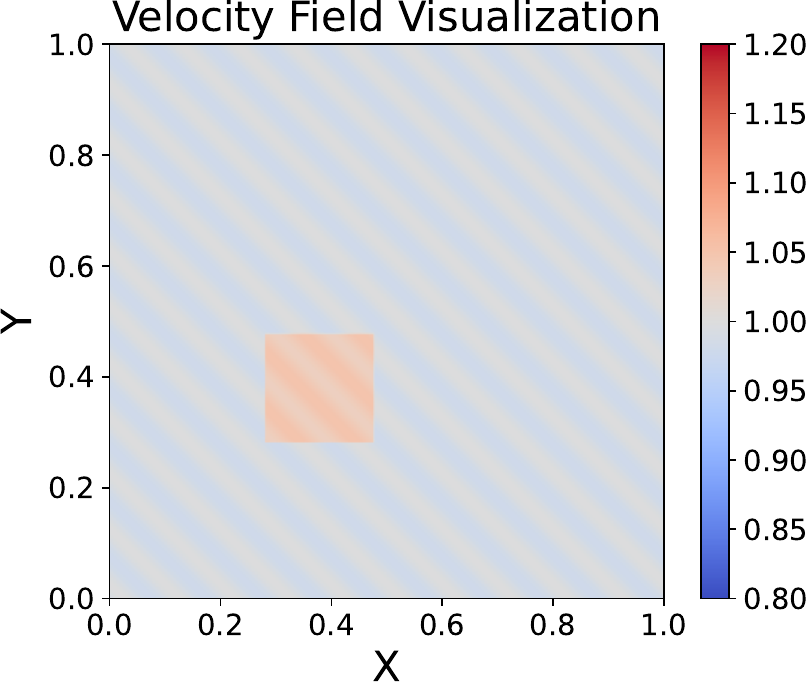} 
\caption{ Small Discontinuity}
\end{subfigure}
\begin{subfigure}[h]{0.24\textwidth}
\centering
\includegraphics[width=\textwidth]{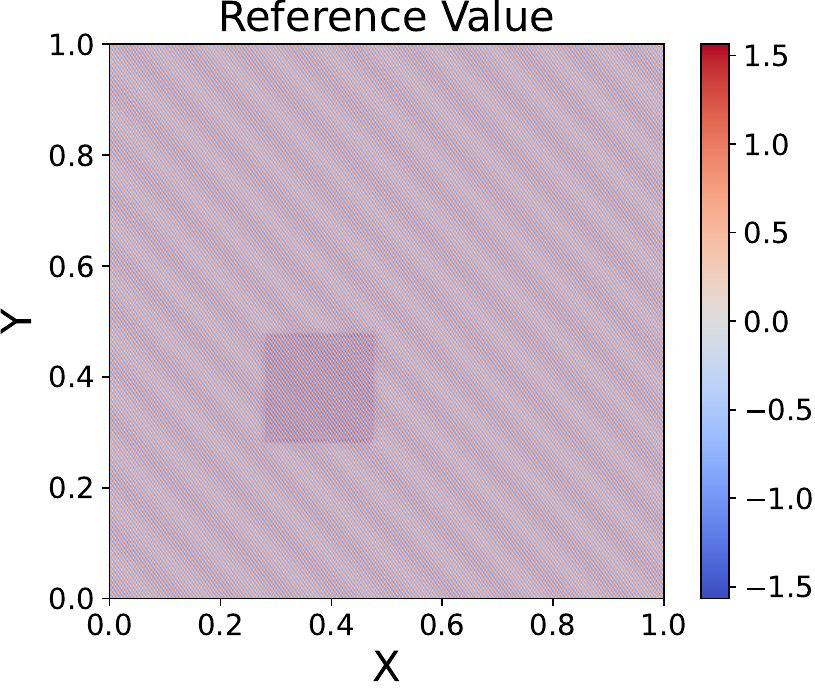} 
\caption{Reference (Small)}
\end{subfigure}
\hfill
\begin{subfigure}[h]{0.24\textwidth}
\centering
\includegraphics[width=\textwidth]{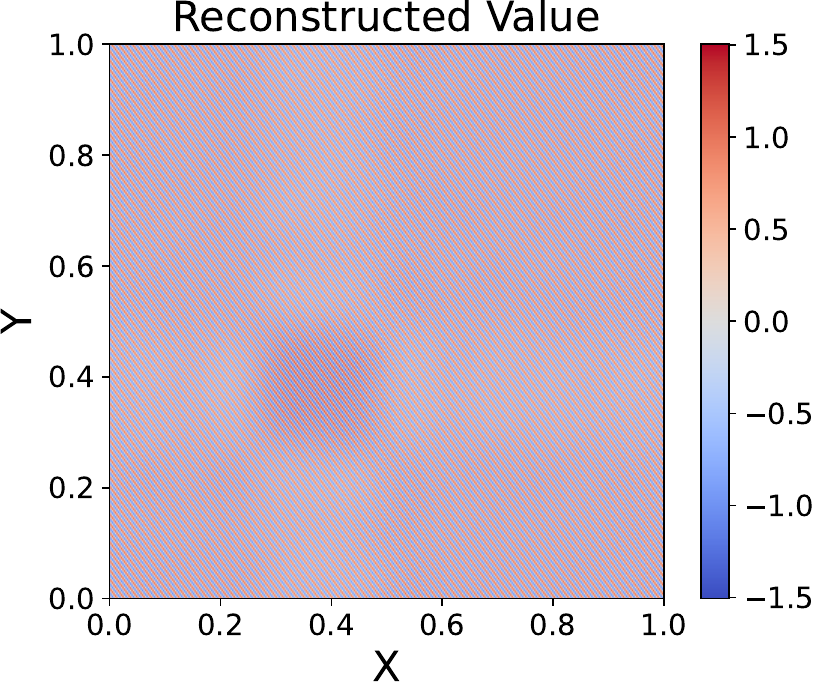} 
\caption{WFP Prediction (Small)}
\end{subfigure}
\hfill
\begin{subfigure}[h]{0.24\textwidth}
\centering
\includegraphics[width=\textwidth]{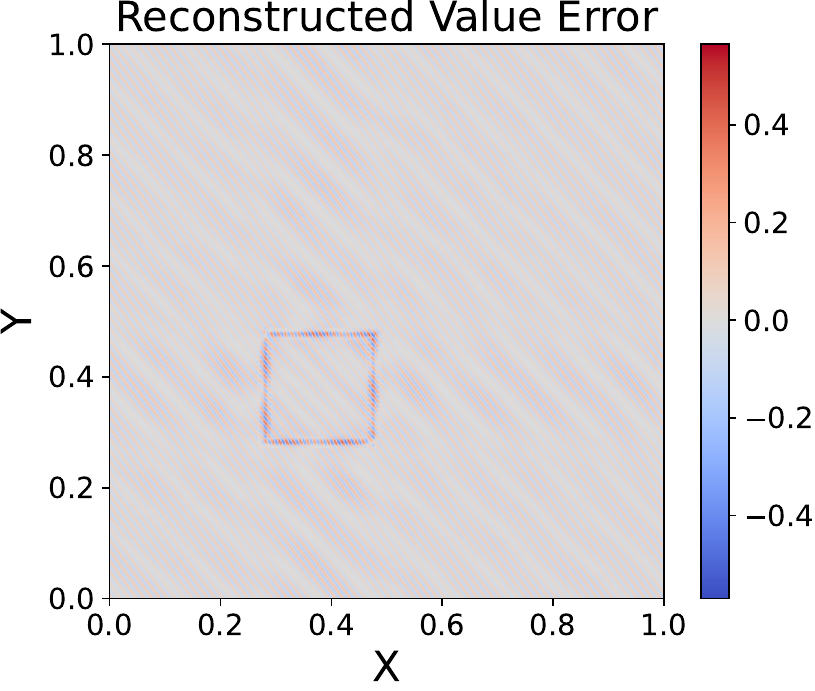} 
\caption{Error (Small)}
\end{subfigure}
\vspace{1em}

\begin{subfigure}[h]{0.24\textwidth}
\centering
\includegraphics[width=\textwidth]{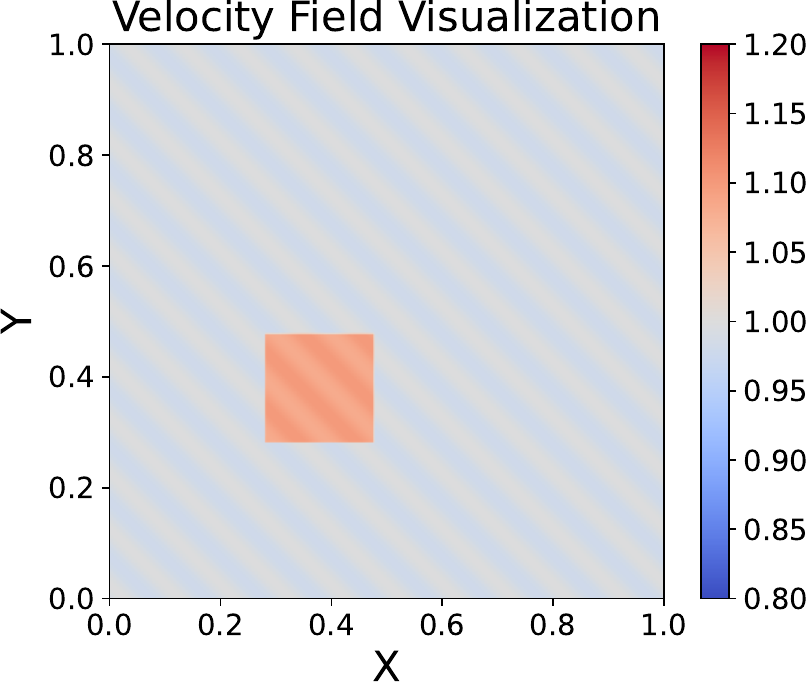} 
\caption{Medium Discontinuity}
\end{subfigure}
\begin{subfigure}[h]{0.24\textwidth}
\centering
\includegraphics[width=\textwidth]{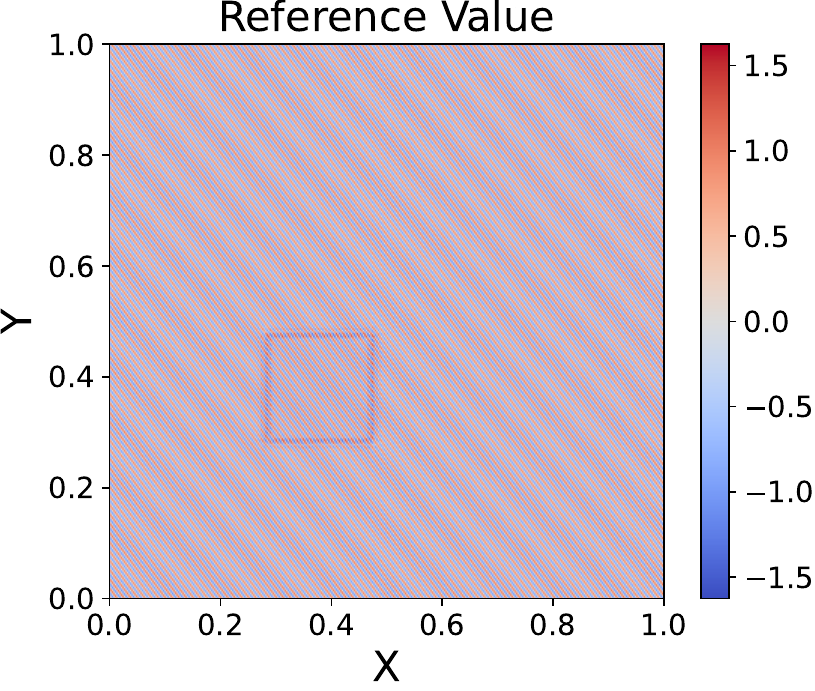} 
\caption{Reference (Medium)}
\end{subfigure}
\hfill
\begin{subfigure}[h]{0.24\textwidth}
\centering
\includegraphics[width=\textwidth]{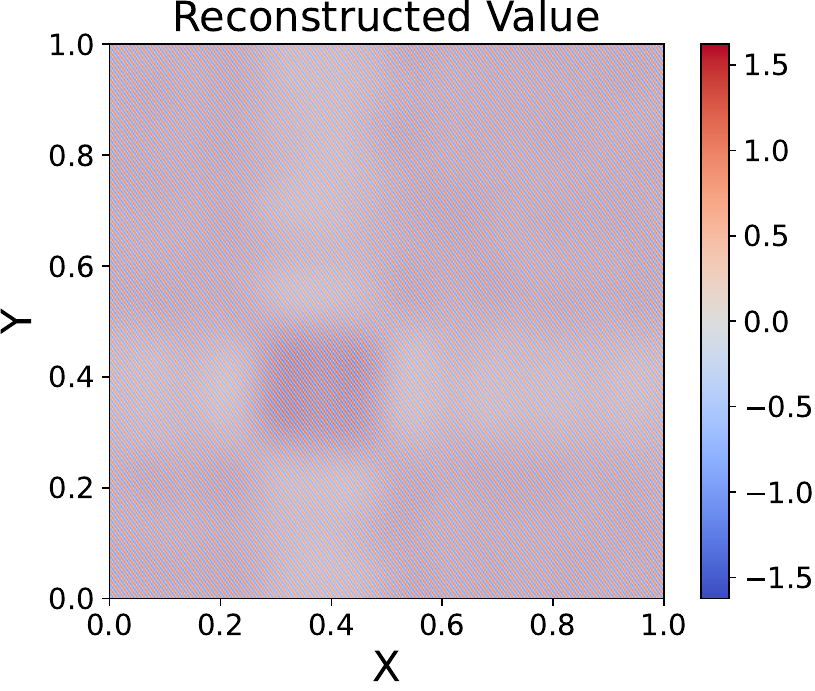} 
\caption{WFP Prediction (Medium)}
\end{subfigure}
\hfill
\begin{subfigure}[h]{0.24\textwidth}
\centering
\includegraphics[width=\textwidth]{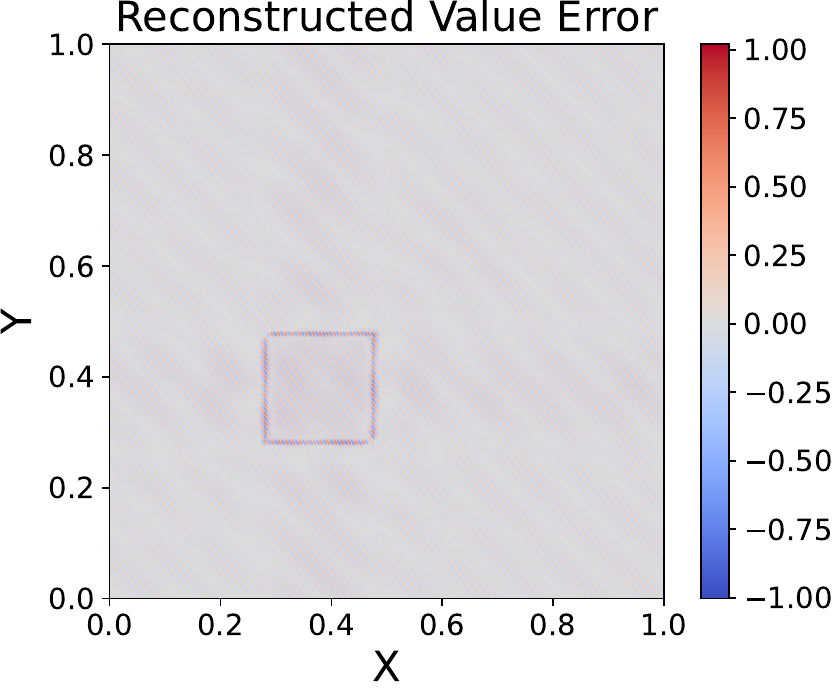} 
\caption{Error (Medium)}
\end{subfigure}
\vspace{1em}

\begin{subfigure}[h]{0.24\textwidth}
\centering
\includegraphics[width=\textwidth]{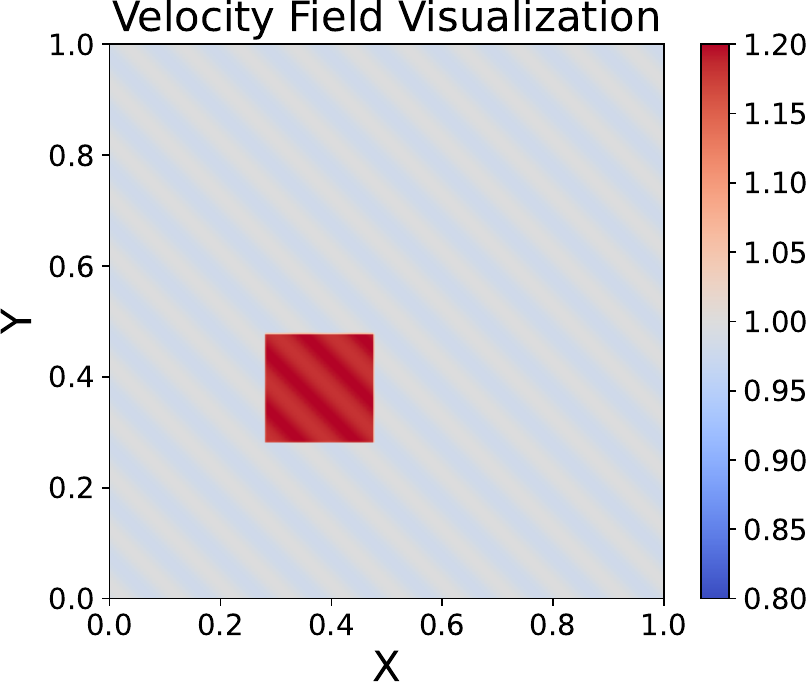} 
\caption{Large Discontinuity}
\end{subfigure}
\begin{subfigure}[h]{0.24\textwidth}
\centering
\includegraphics[width=\textwidth]{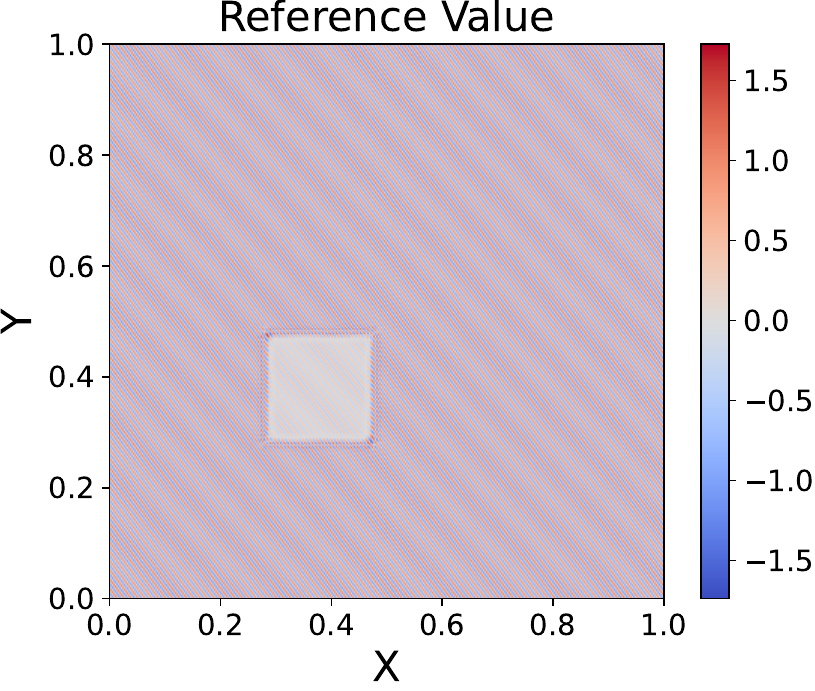} 
\caption{Reference (Large)}
\end{subfigure}
\hfill
\begin{subfigure}[h]{0.24\textwidth}
\centering
\includegraphics[width=\textwidth]{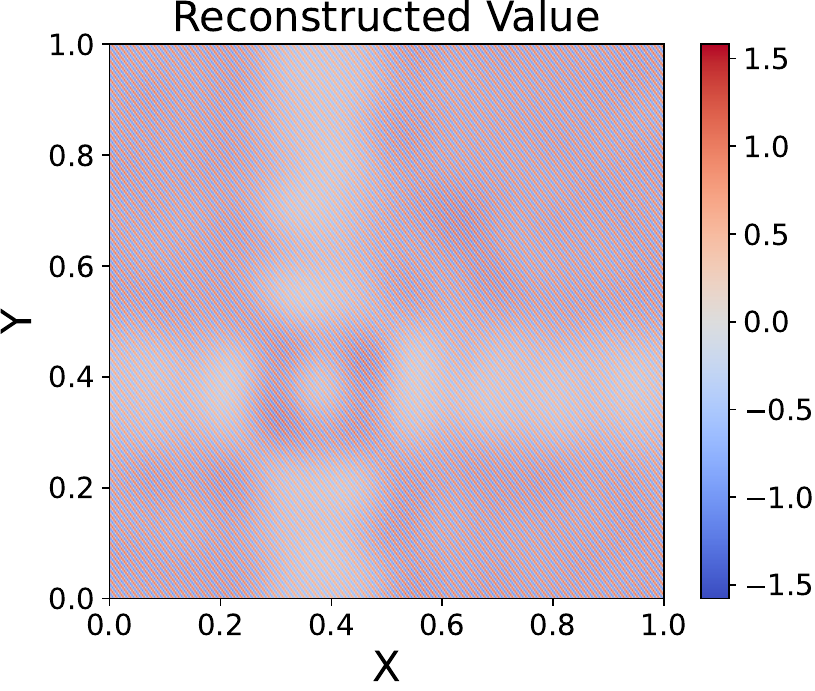} 
\caption{WFP Prediction (Large)}
\end{subfigure}
\hfill
\begin{subfigure}[h]{0.24\textwidth}
\centering
\includegraphics[width=\textwidth]{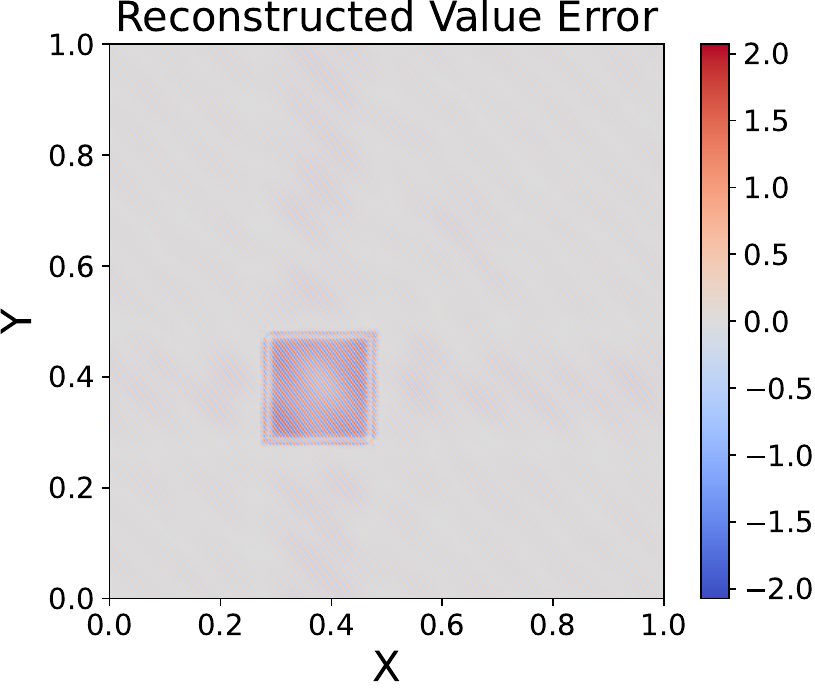} 
\caption{Error (Large)}
\end{subfigure}
\caption{OOD Discontinuity Generalization: Performance on a velocity field with a sharp discontinuity. The model handles small jumps well but struggles as the discontinuity strength increases.}
\label{fig:discontinuous_media_ood}
\end{figure}

\subsection{Wave Equation - Out of Distribution Generalization}\label{sec:ood_generalization}

A critical test for any physics-informed neural operator is its ability to generalize to scenarios that lie outside its training distribution. In this section, we evaluate the WFP's robustness on two challenging out-of-distribution (OOD) tasks: wave propagation in media with sharp discontinuities and in media with perturbation strengths beyond those seen during training.

We first test the model performance on a velocity field with a sharp discontinuity in Figure~\ref{fig:discontinuous_media_ood}. The network demonstrates robust and predictable behavior. For small discontinuities, the error is minimal and tightly localized at the boundary of discontinuity, with the solution in the smooth regions remaining highly accurate. As the discontinuity strength increases, the error grows predictably and begins to propagate into the interior, likely as the local wave speed exceeds the range seen during training. This graceful degradation highlights the stability of our method: while not unconditionally generalizable to arbitrary velocity fields, it performs reliably on weakly discontinuous media and avoids catastrophic failure in strongly out-of-distribution scenarios.

Next, we investigate how the model extrapolates to velocity fields with perturbation strengths greater than those encountered during training. We generate smooth media but scale the perturbation term by a factor that pushes it beyond the training distribution, as shown in Figure~\ref{fig:stronger_perturbation_ood}. The results reveal a similar pattern of graceful degradation. For perturbations slightly outside the training range, the WFP's predictions remain accurate, with errors primarily concentrated in regions where the local velocity contrast is highest. As the perturbation strength grows, the error increases more rapidly but stays localized in regions of high velocity contrast. This behavior is quantitatively summarized in Figure~\ref{fig:perturbation_error_curve}, which plots the L1 error as a function of the perturbation scaling factor. The curve shows a slow, near-linear increase in error for moderately OOD inputs, followed by a steeper rise for far-OOD scenarios. This confirms that while the WFP has a reliable, albeit limited, range of extrapolation, its performance degrades predictably rather than catastrophically.

\begin{figure}
\centering
\begin{subfigure}[h]{0.24\textwidth}
\centering
\includegraphics[width=\textwidth]{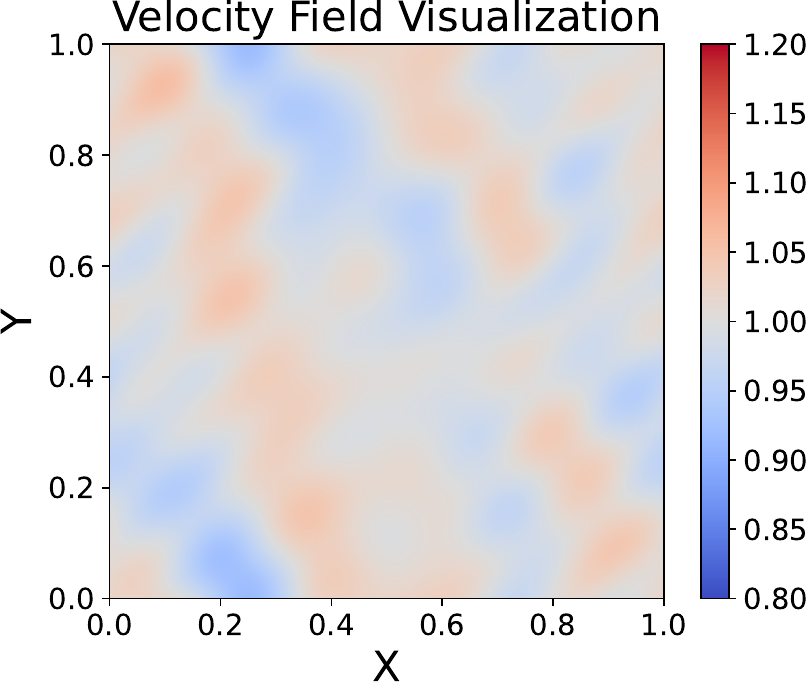} 
\caption{Small Disturbance}
\end{subfigure}
\begin{subfigure}[h]{0.24\textwidth}
\centering
\includegraphics[width=\textwidth]{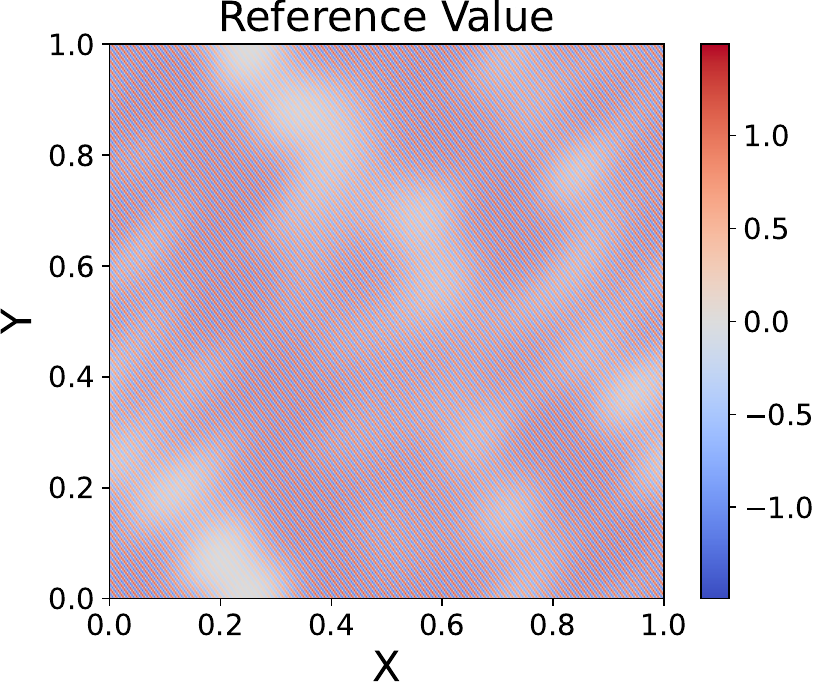} 
\caption{Reference (Small)}
\end{subfigure}
\hfill
\begin{subfigure}[h]{0.24\textwidth}
\centering
\includegraphics[width=\textwidth]{figures/disturbance_0.05_reconstructed_value.pdf} 
\caption{WFP Prediction (Small)}
\end{subfigure}
\hfill
\begin{subfigure}[h]{0.24\textwidth}
\centering
\includegraphics[width=\textwidth]{figures/disturbance_0.05_reconstructed_value_error.pdf} 
\caption{Error (Small)}
\end{subfigure}
\vspace{1em}

\begin{subfigure}[h]{0.24\textwidth}
\centering
\includegraphics[width=\textwidth]{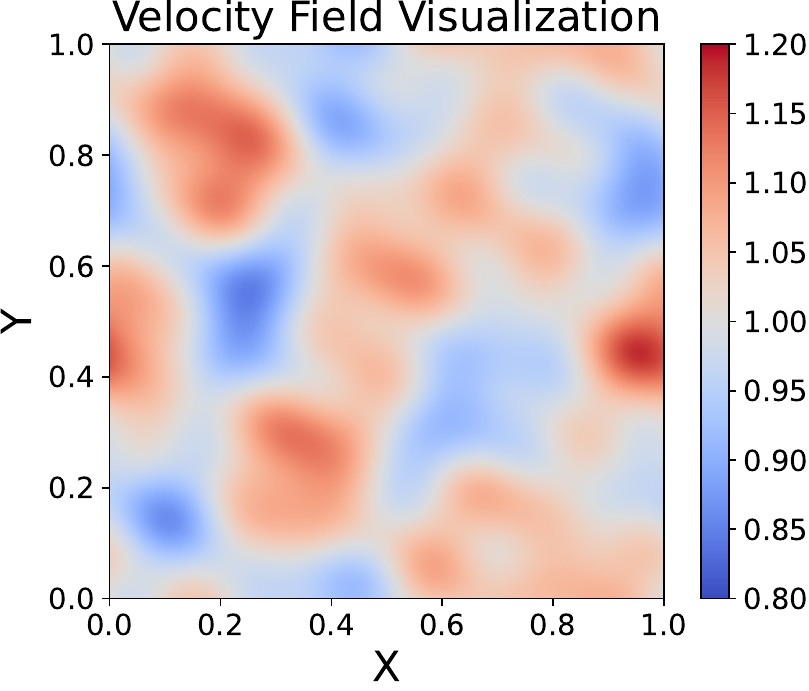} 
\caption{Medium Disturbance}
\end{subfigure}
\begin{subfigure}[h]{0.24\textwidth}
\centering
\includegraphics[width=\textwidth]{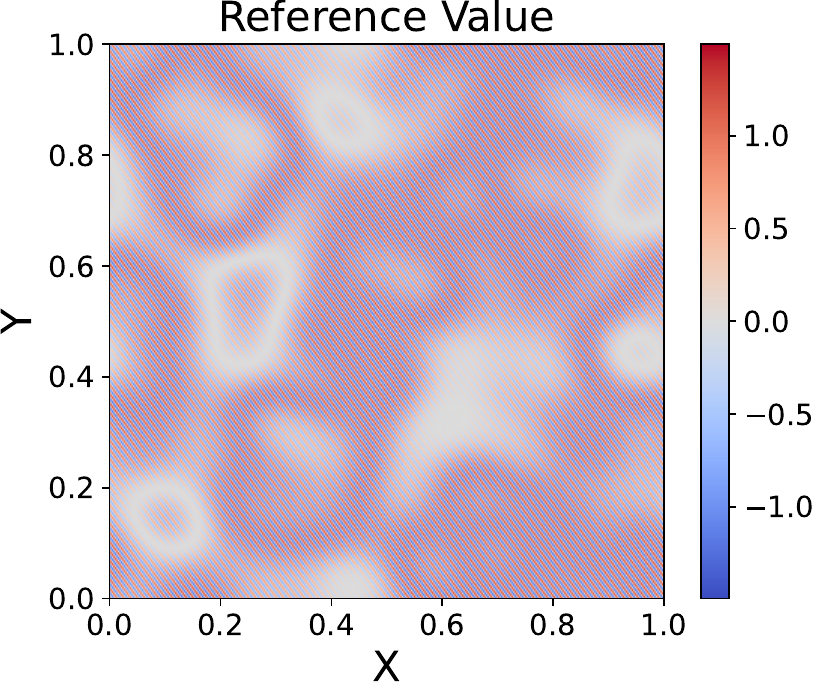} 
\caption{Reference (Medium)}
\end{subfigure}
\hfill
\begin{subfigure}[h]{0.24\textwidth}
\centering
\includegraphics[width=\textwidth]{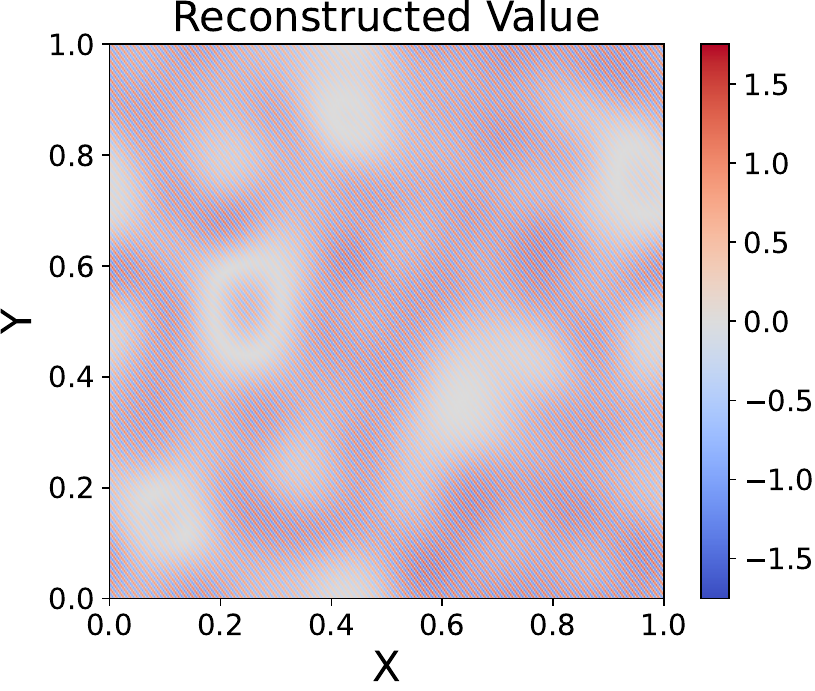} 
\caption{WFP Prediction (Medium)}
\end{subfigure}
\hfill
\begin{subfigure}[h]{0.24\textwidth}
\centering
\includegraphics[width=\textwidth]{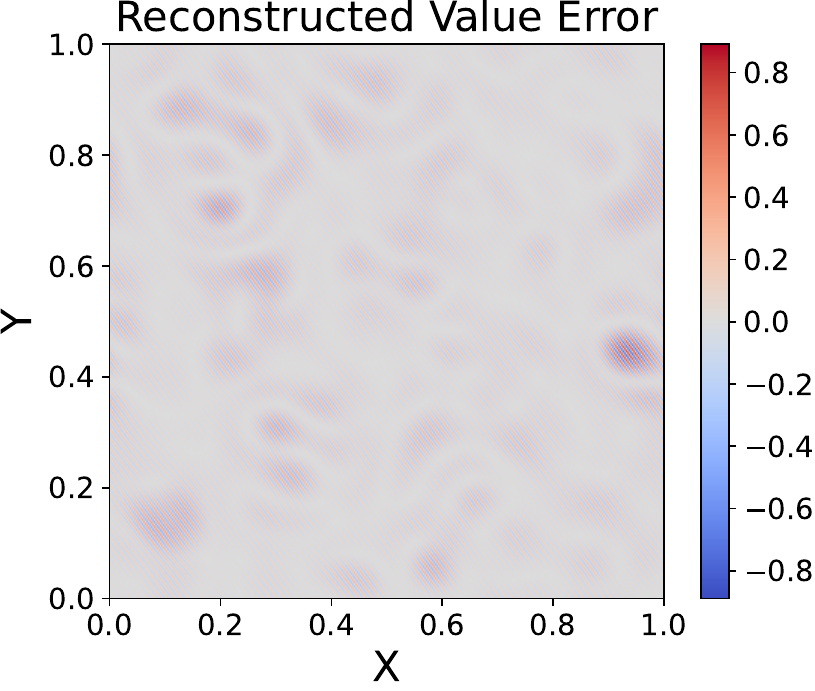} 
\caption{Error (Medium)}
\end{subfigure}
\vspace{1em}

\begin{subfigure}[h]{0.24\textwidth}
\centering
\includegraphics[width=\textwidth]{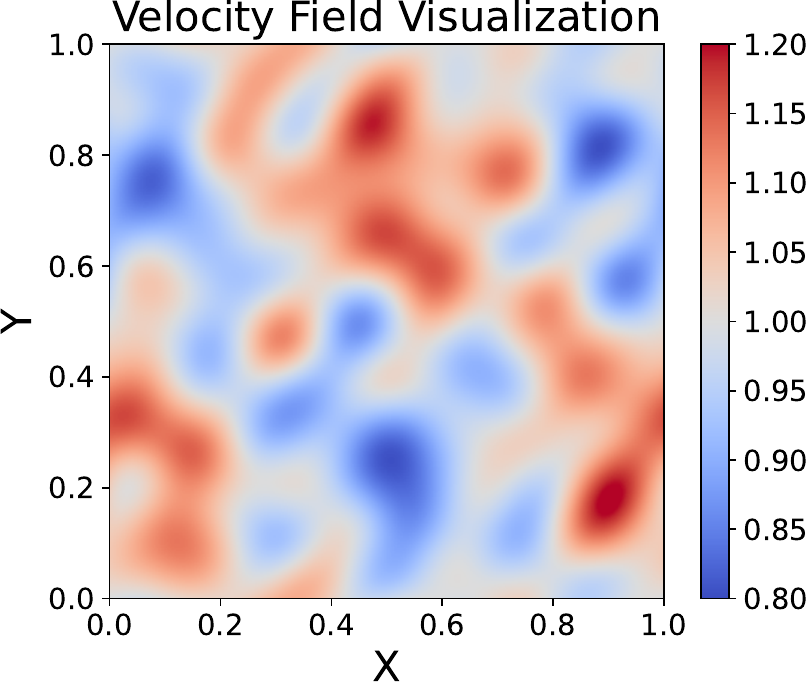} 
\caption{Large Disturbance}
\end{subfigure}
\begin{subfigure}[h]{0.24\textwidth}
\centering
\includegraphics[width=\textwidth]{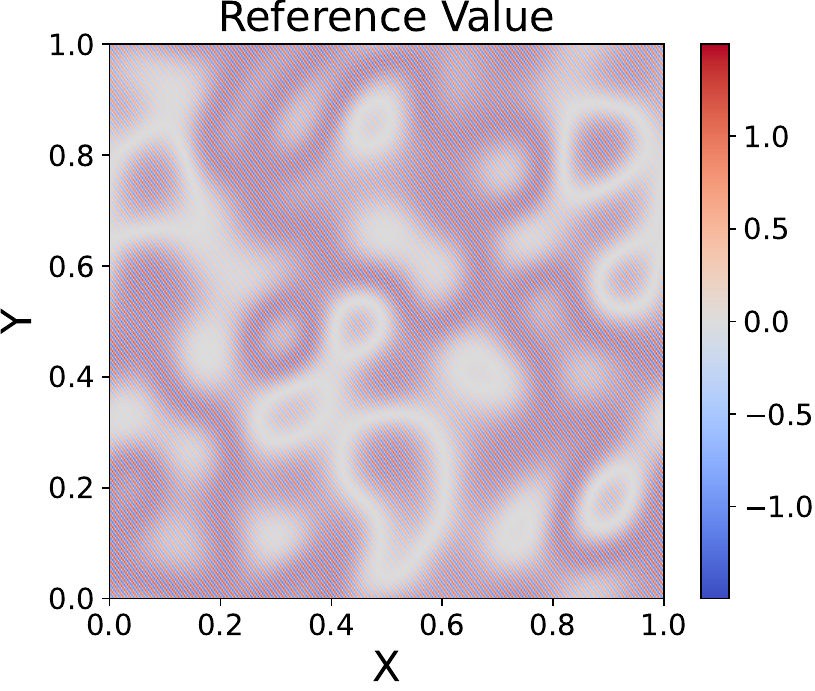} 
\caption{Reference (Large)}
\end{subfigure}
\hfill
\begin{subfigure}[h]{0.24\textwidth}
\centering
\includegraphics[width=\textwidth]{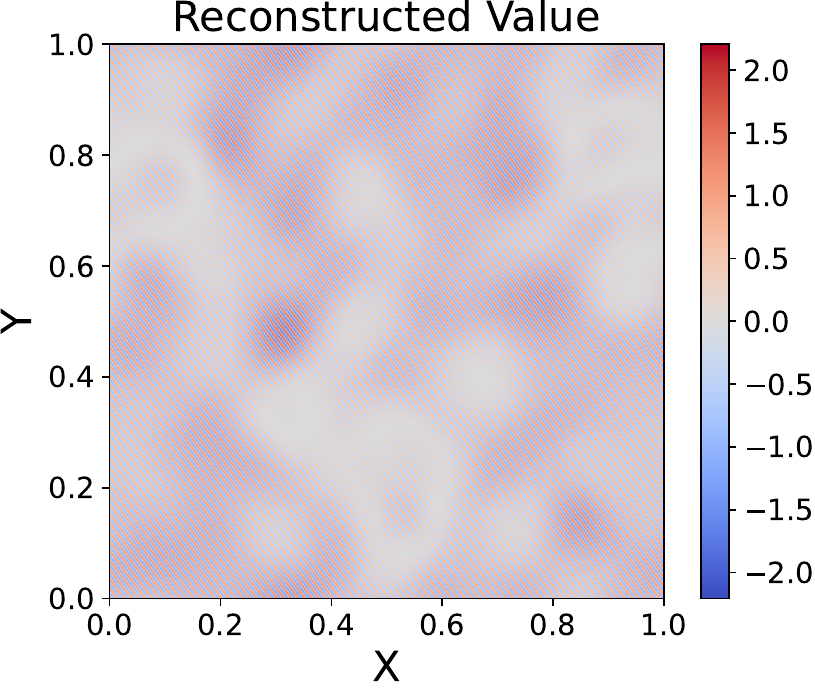} 
\caption{WFP Prediction (Large)}
\end{subfigure}
\hfill
\begin{subfigure}[h]{0.24\textwidth}
\centering
\includegraphics[width=\textwidth]{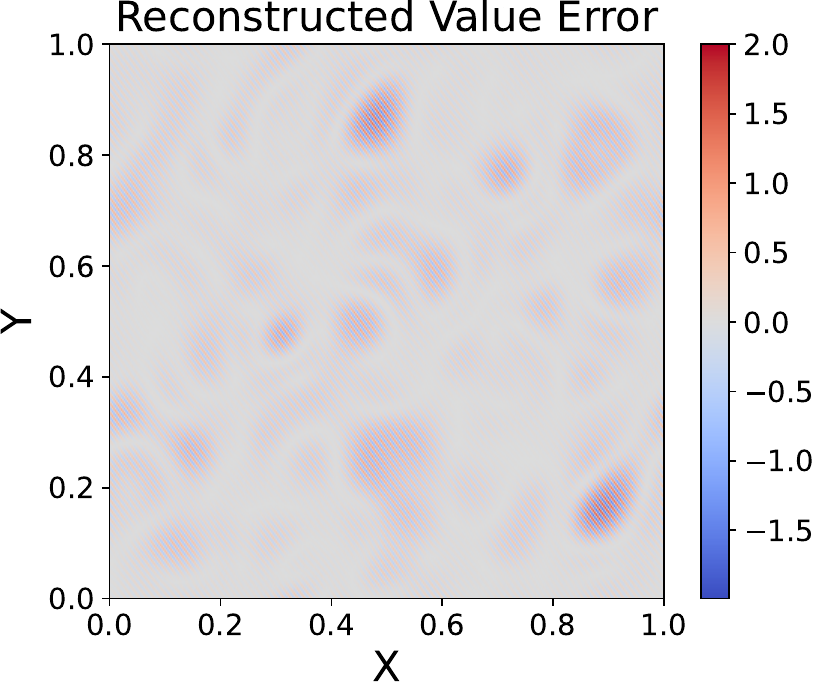} 
\caption{Error (Large)}
\end{subfigure}


\caption{OOD Disturbance Generalization: Performance on velocity fields with perturbation strengths beyond the training distribution. The model shows graceful and predictable degradation as the input moves further from the training data.}
\label{fig:stronger_perturbation_ood}
\end{figure}
\begin{figure}[h]
	\centering
	\includegraphics[width=0.8\textwidth]{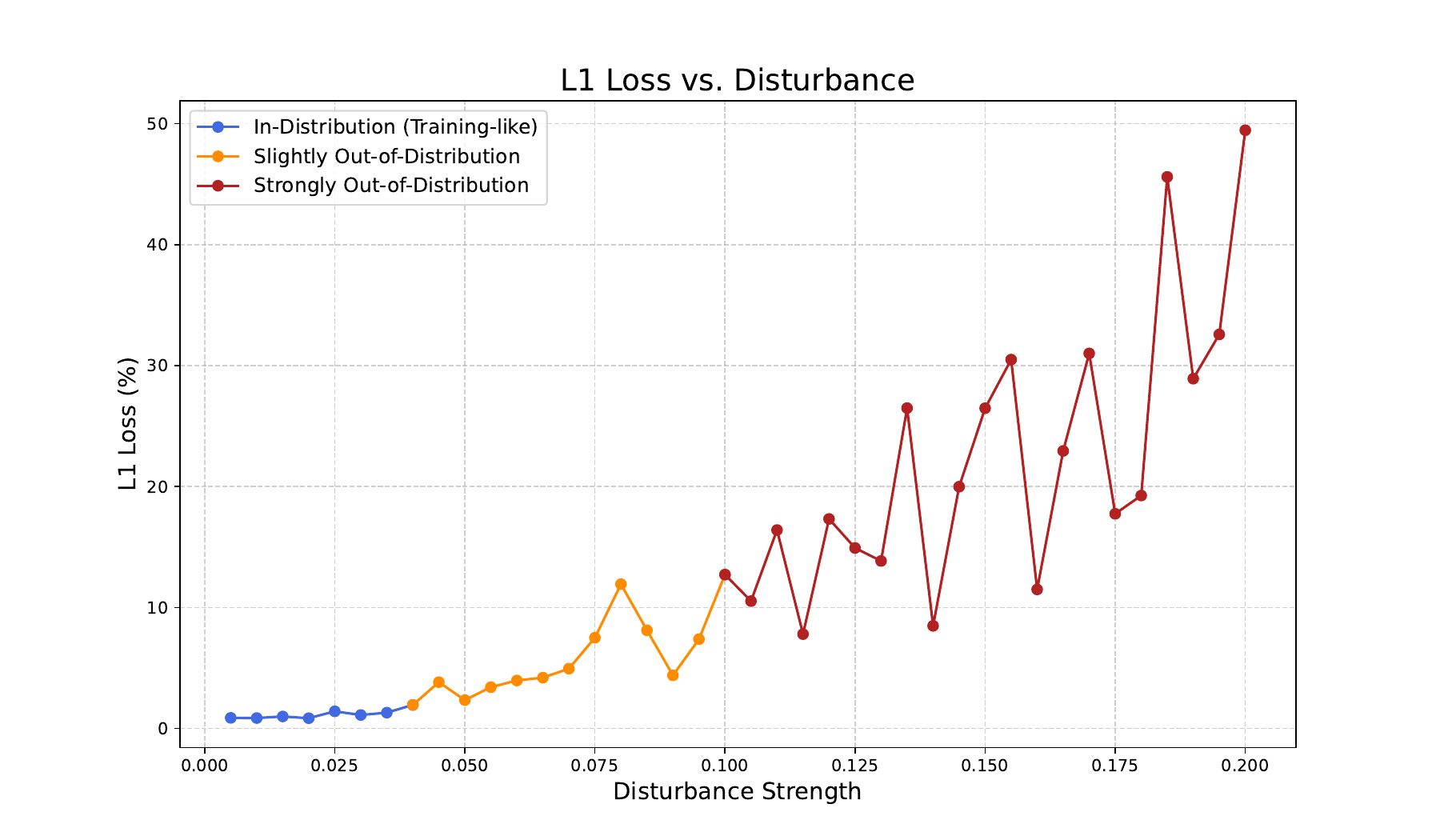} 
	\caption{L1 Error vs. Disturbance Strength: Error increases predictably as the perturbation strength exceeds the training distribution, demonstrating the model's limited but reliable extrapolation capabilities.}
	\label{fig:perturbation_error_curve}
\end{figure}

\subsection{Ablation Study}\label{ablation_study}

\textbf{Activation Choice.}
One of the key factors contributing to the strong fitting performance 
and generalization capability is careful hyperparameter tuning, 

For this specific task, 
commonly used activation functions such 
as $\mathrm{ReLU}$ and $\tanh$ 
still have room for improvement. 
An activation function incorporating 
trigonometric terms was adopted, 
which led to significantly improved performance,
as some related work has shown~\cite{sitzmann2020implicit}. 
This outcome can be attributed to the inherent 
periodic variation of the numerical characteristics 
in wave equations with respect to frequency, 
a property that conventional activation functions 
are ill-suited to capture effectively, as demonstrated in~\cite{tancik2020fourier,xu2020frequency}.

\textbf{Impact of Activation Function Selection on Generalization and Error:}
The choice of activation function for the main processing branch of our network is one contributing factor to its ability to learn the complex, oscillatory dynamics of the wave equation. 
We conducted an ablation study (Table~\ref{tab:activation_ablation}) 
to compare the performance of standard activation functions 
against a custom, physics-inspired alternative. 
See detailed loss curves in Figure~\ref{fig:activation_function} 
in Appendix~\ref{appendix:activation_functions}.

\begin{table}[h]
	\begin{center}
		\caption{Ablation Study: Impact of Activation Function on Training and Test Loss.\label{tab:activation_ablation}}
		\begin{tabular}{lcc}
			\multicolumn{1}{c}{\bf Activation Function} & \multicolumn{1}{c}{\bf Training Loss (MSE)} & \multicolumn{1}{c}{\bf Test Loss (MSE)} \\ \hline
			$\tanh$ & $1 \times 10^{-4}$ & $2 \times 10^{-4}$ \\
			Leaky ReLU (slope=0.05) & $4 \times 10^{-5}$ & $2 \times 10^{-4}$ \\
			$\exp(-4x^2)\sin(3\pi x)$ & $5 \times 10^{-5}$ & $6 \times 10^{-5}$ \\
			$\exp(-20x^2)\sin(10\pi x)$  & $\mathbf{4 \times 10^{-6}}$ & $\mathbf{8 \times 10^{-6}}$ \\
		\end{tabular}
	\end{center}
\end{table}

\section{Conclusion and Discussion}\label{sec:conclusion}

This paper introduces the Windowed Fourier Propagator (WFP), 
a novel deep learning framework designed to overcome the challenges of solving 
high-frequency wave equations in inhomogeneous media. 
By synergizing fundamental principles of wave physics with a specialized neural operator architecture, 
the WFP establishes a new paradigm for efficient and accurate wave propagation modeling. 
Our mathematical analysis and extensive numerical validation have demonstrated that the WFP not only achieves exceptional computational efficiency but also exhibits robust generalization capabilities.

The core contributions of this work are rooted in its physics-informed design. 
First, we leveraged the principle of frequency locality to construct a computationally 
tractable yet expressive model that captures essential wave spreading phenomena 
without incurring the quadratic complexity of dense interaction models. 
Second, by explicitly preserving the superposition principle, 
our architecture demonstrates remarkable generalization to arbitrary initial conditions, 
even when trained on a simple basis of plane waves. 
Third, the WFP's operation in the frequency domain ensures resolution-invariance, 
making it a versatile tool for multi-resolution analysis.

The effectiveness of this framework translates directly into significant practical advantages. 
The WFP drastically accelerates the solution of forward wave problems and enhances the efficiency 
of gradient-based optimization for inverse problems. 
In conclusion, by demonstrating the value of equation-specific network design, 
this work paves the way for a new class of resource-efficient and physically-consistent 
machine learning applications for complex PDE systems.

\section*{Acknowledgments}
We thank Shanghai Institute for Mathematics and Interdisciplinary Sciences (SIMIS) for their financial support. This research was funded by SIMIS under grant number SIMIS-ID-2024-(CN). The authors are grateful for the resources and facilities provided by SIMIS, which were essential for the completion of this work.

\appendix
\section{Neural Network Architecture Details}\label{appendix:nn_operations}
This appendix details the components of the Gated Neural Network module used for learning the frequency propagator.

\subsection{Fundamental Neural Network Operations}
Our architecture is constructed from three standard operations:
\begin{enumerate}
	\item \textbf{Linear Transformation}: A mapping $\mathbf{y} = \mathbf{W}\mathbf{x} + \mathbf{b}$, denoted $\mathrm{Linear}(h_{in}, h_{out})$, where $\mathbf{W}$ and $\mathbf{b}$ are learnable parameters. $(h_{in}, h_{out})$ are the input and output dimensions, respectively.
	\item \textbf{Activation Function ($\sigma(\cdot)$)}: A non-linear function applied element-wise. We use two types: the standard \textbf{Sigmoid Function} $\sigma_{gate}(x) = (1 + e^{-x})^{-1}$ in the gate branch, and a \textbf{Custom Trigonometric Activation} $\sigma_{main}(x) = A \exp(-Bx^2)\sin(Cx)$ in the main processing path, as detailed in Section~\ref{ablation_study}.
	\item \textbf{Element-wise Product ($\odot$)}: The standard element-wise multiplication of two vectors.
\end{enumerate}

\subsection{The Gated Neural Network Module}
The core of our model is a Gated Neural Network (Figure~\ref{fig:gate_structure}) that maps an input token $\mathbf{z}_{in} \in \mathbb{R}^{h_1}$ to a predicted frequency evolution dictionary. The input token concatenates a vector representing the initial frequency mode with the feature vector extracted from the velocity field. The token is processed through two parallel branches:

\begin{enumerate}
	\item \textbf{Main Branch}: The input $\mathbf{z}_{in}$ is passed through two linear layers interspersed with our custom activation function $\sigma_{main}$ to produce the main output $\mathbf{o}_{\text{main}}$.
	\[
	\mathbf{a}_{\text{hidden}} = \sigma_{main}(\mathrm{Linear}_1(\mathbf{z}_{in})) \quad \implies \quad \mathbf{o}_{\text{main}} = \mathrm{Linear}_2(\mathbf{a}_{\text{hidden}})
	\]
	\item \textbf{Gate Branch}: The same input $\mathbf{z}_{in}$ is passed through a single linear layer followed by the sigmoid activation $\sigma_{gate}$ to produce a modulation vector $\mathbf{o}_{\text{gate}}$.
	\[
	\mathbf{o}_{\text{gate}} = \sigma_{gate}(\mathrm{Linear}_3(\mathbf{z}_{in}))
	\]
	\item \textbf{Gating Operation}: The final output is the element-wise product of the two branches:
	\[
	\mathbf{y}_{\text{final}} = \mathbf{o}_{\text{main}} \odot \mathbf{o}_{\text{gate}}
	\]
\end{enumerate}
This gated mechanism allows the network to learn a data-driven filter, selectively amplifying or suppressing components of the learned transformation. This enforces a soft sparsity, focusing the model on the most physically relevant features of the wave propagation problem.

\section{Training details and Loss Curves}\label{appendix:loss_curves}
\begin{figure}[h]
	\centering
	\includegraphics[width=0.5\textwidth]{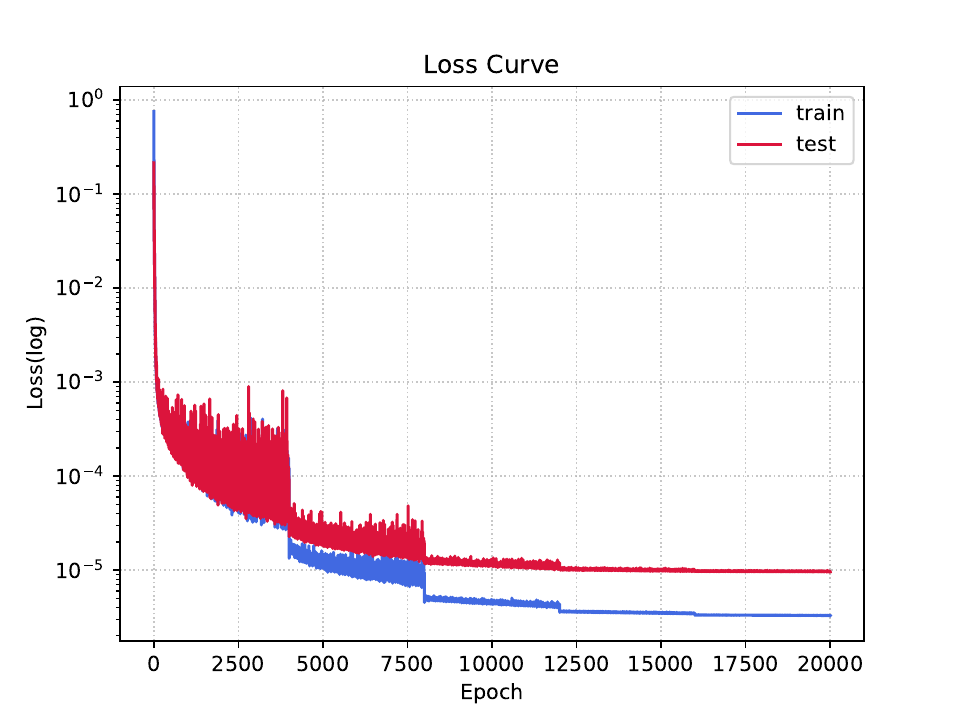}
	
	\raggedright
	
	\textbf{Network Training Loss Visualization.}
	Training and testing loss curves are plotted for the 1D wave equation.
	\caption{Visualization of 1D Network Training Results}
	\label{fig:training_loss}
\end{figure}

\subsection{Wave Equation Training Details.}
For numerical experiments, 
the training set comprises $10,000$ samples for 1D situation and $100,000$ samples for 2D situation,
with a test set of $200$ samples. 
The error metric (y-axis) is quantified using Mean Square Error (MSE). 
Training is conducted with a batch size of $100$ 
over $20,000$ epochs.
An initial learning rate of $10^{-3}$ is implemented, 
incorporating a decay factor of $0.1$ every $4,000$ steps,
and the network architecture includes a hidden layer with $6,000$ neurons for 
1D situation and $8,000$ neurons for 2D situation.
The activation function is defined as 
$\exp(-20x^2)\sin(10\pi x)$.
In the FFT image segmentation of initial values,
one element is selected each time,
combined with $10^d$ velocity main frequencies,
to output $15^d$ surrounding frequencies at $t=0.02$ seconds.
No additional regularization techniques are employed.

\subsection{Ablation Study on Activation Functions.}\label{appendix:activation_functions}

\begin{figure}[!h]
	\centering
	\begin{subfigure}[htbp]{0.45\textwidth}
		\centering
		\includegraphics[width=\textwidth]{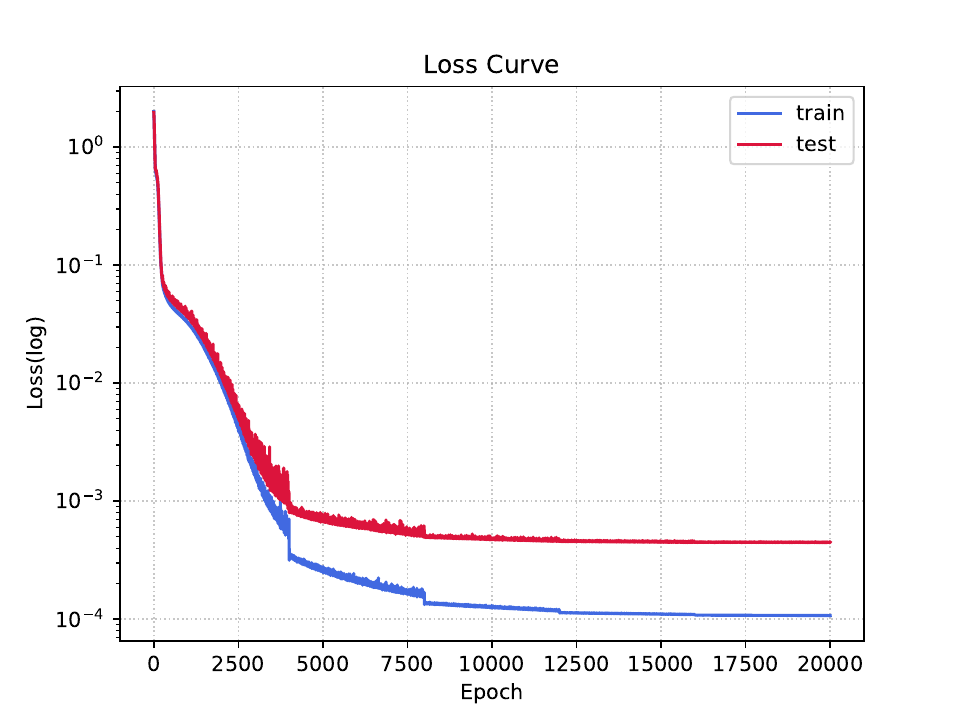}
		\caption{$\tanh(|x|)$}
		\label{fig:act_tanh}
	\end{subfigure}
	\hfill
	\begin{subfigure}[htbp]{0.45\textwidth}
		\centering
		\includegraphics[width=\textwidth]{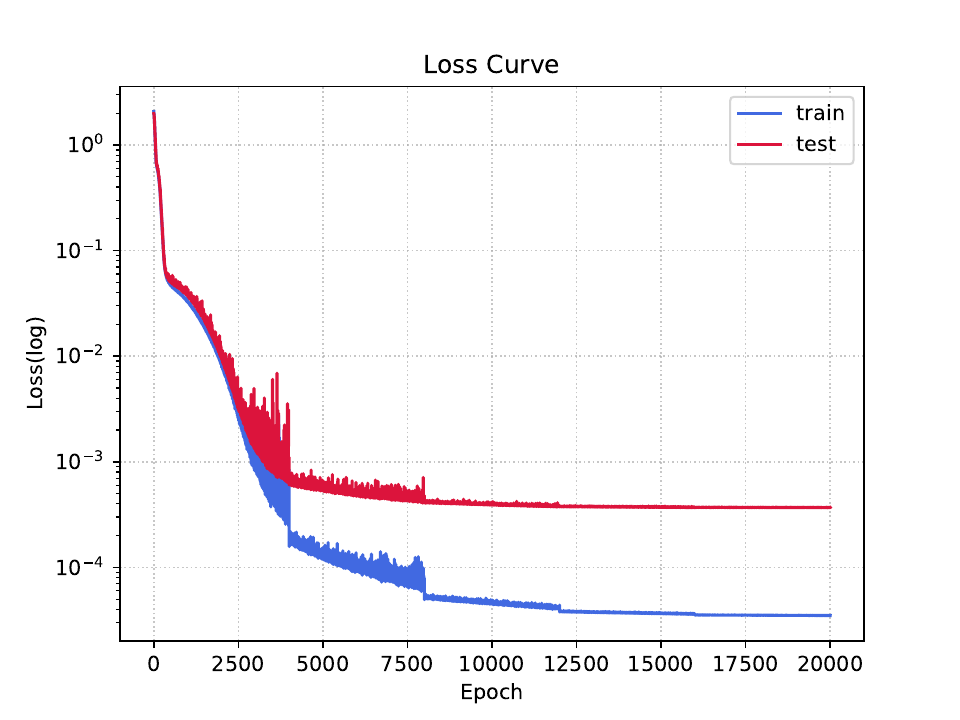}
		\caption{Leaky$(0.05)$ReLU(x)}
		\label{fig:act_relu}
	\end{subfigure}
	\vskip\baselineskip
	\begin{subfigure}[htbp]{0.45\textwidth}
		\centering
		\includegraphics[width=\textwidth]{figures/training_loss_3.pdf}
		\caption{$\exp(-20x^2)\sin(10\pi x)$}
		\label{fig:act_exp20}
	\end{subfigure}
	\hfill
	\begin{subfigure}[htbp]{0.45\textwidth}
		\centering
		\includegraphics[width=\textwidth]{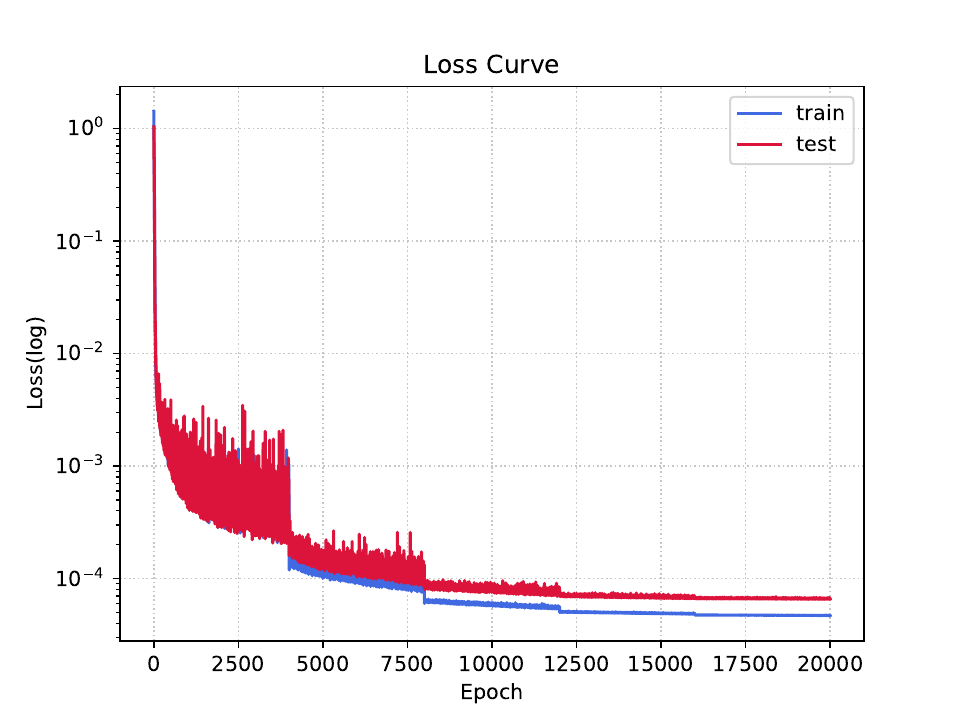}
		\caption{$\exp(-4x^2)\sin(3\pi x)$}
		\label{fig:act_exp4}
	\end{subfigure}
	
	\raggedright
	
	\caption{Effect of Different Activation Functions on Network Training Results}
	\label{fig:activation_function}
\end{figure}

Our experiments show that while conventional activation functions such as 
$\mathrm{\tanh}$ (Figure~\ref{fig:act_tanh}) and $\mathrm{Leaky ReLU}$ 
(Figure~\ref{fig:act_relu}) provide a reasonable performance baseline, 
we found that accuracy could be significantly enhanced by employing a
custom activation function tailored to the problem's oscillatory nature.

To better align the network's inductive bias with the 
problem's oscillatory nature, we introduced a custom activation 
function of the form $\exp(-Ax^2)\sin(Bx)$. 
Through hyperparameter search, we identified an optimal set of parameters,
yielding the function $\exp(-20x^2)\sin(10\pi x)$, 
which demonstrates superior performance 
(Figure~\ref{fig:act_exp20}). 
This specifically tuned function achieves the lowest 
error rates and exhibits the best generalization 
capabilities among all tested configurations.

Furthermore, we observed that even with a non-optimal parameter set, 
such as in the function 
$\exp(-4x^2)\sin(3\pi x)$ (Figure~\ref{fig:act_exp4}), 
or set $A,B$ as trainable parameters,
our custom activation function architecture still consistently outperforms both 
$\mathrm{\tanh}$ and $\mathrm{Leaky ReLU}$. 
This indicates that the structural form of combining a sinusoidal 
component with a Gaussian envelope is inherently better suited for this class of problems, 
providing a robust performance advantage over standard activation functions, 
with careful parameter tuning offering further significant improvements.

\subsection{Benchmark Tests}
There are two primary FNO structures, 
distinguished by the form of the $R$ matrix. 
The $R$ matrix can be implemented either as an elementwise operation 
or as a linear transformation.
We set the frequency threshold to $128$ in 
the one-dimensional FNO, as the highest frequency of input data is $96$ ($\sin(2\pi\times 96x)$).

\begin{itemize}
	\item \textbf{Elementwise (or Diagonal) $R$:} In this configuration, 
	$R$ is a diagonal matrix 
	(or, equivalently, a vector that multiplies the Fourier coefficient vector elementwise). 
	The number of learnable parameters is proportional to the number of retained Fourier modes.
	\item \textbf{Linear Transformation (or Dense) $R$:} Here, $R$ is a fully connected linear transformation matrix,
	enabling more complex interactions among different frequency modes. 
	The number of parameters scales quadratically with the number of retained modes.
\end{itemize}

For illustration, consider the one-dimensional case where $100$ effective frequencies are retained. 
The elementwise $R$ comprises $100$ learnable complex parameters, 
whereas the linear transformation $R$ consists of a $100 \times 100$ complex-valued matrix, 
totaling $10,000$ parameters.

A notable challenge arises from the curse of dimensionality: in higher-dimensional settings, 
the parameter count for the linear transformation matrix becomes prohibitively large, 
rendering training infeasible. 
For instance, retaining $100 \times 100$ frequencies in two dimensions results in 
a linear transformation matrix with $(100 \times 100)^2$ parameters,
which becomes infeasible for training.

Consequently, comparative experiments were conducted exclusively in the one-dimensional setting, 
evaluating both FNO with elementwise $R$ and linear $R$.

\begin{table}
	\begin{center}

		\caption{Comparison of training and test loss for different 
			models on the 1D wave equation.\label{tab:1d_comparison}}
		
		\begin{tabular}{lcc}
			\multicolumn{1}{c}{\bf MODEL}  & \multicolumn{1}{c}{\bf Training Loss (MSE)} & \multicolumn{1}{c}{\bf Test Loss (MSE)} \\ \hline
			FNO (Elementwise $R$) & $4 \times 10^{-4}$ & $4 \times 10^{-4}$ \\
			FNO (Linear $R$)      & $6 \times 10^{-5}$ & $7 \times 10^{-5}$ \\
			\textbf{WFP (Ours)}   & $\mathbf{4 \times 10^{-6}}$ & $\mathbf{8 \times 10^{-6}}$ \\
		\end{tabular}
	\end{center}
	
\end{table}
\printcredits
\newpage
\bibliographystyle{cas-model2-names}

\bibliography{refs}



\end{document}